\documentclass[11pt]{article} %

\usepackage{hyperref}
\usepackage{url}

\usepackage[many]{tcolorbox}

\usepackage{algorithm,algorithmicx,algpseudocode}
\usepackage{varwidth}%

\usepackage{hyperref}
\hypersetup{
    colorlinks,
    linkcolor={blue!80!black},
    citecolor={green!50!black},
}
\usepackage{xcolor}
\colorlet{linkequation}{blue}

\usepackage{comment}
\usepackage{latexsym}
\usepackage{amsmath}
\usepackage{amssymb}
\usepackage{amsthm}
\usepackage{epsfig}
\usepackage{graphicx}
\usepackage{bm}
\usepackage[shortlabels]{enumitem}
\usepackage{yfonts}
\usepackage[toc,page]{appendix}
\usepackage{graphicx}
\usepackage{bbm}
\usepackage{accents}
\usepackage{mathtools}
\usepackage{fancybox}

\renewcommand{\P}{\mathbb{P}}
\newcommand{\E}{\mathbb{E}}

\newcommand{\cN}{\mathcal{N}}

\newcommand{\R}{\mathbb{R}}

\newcommand{\N}{\mathbb{N}}

\newcommand{\tr}{\text{tr}}

\newcommand{\error}{\mathrm{error}}

\newtheoremstyle{myremark} %
    {\topsep}                    %
    {\topsep}                    %
    {\rm}                        %
    {}                           %
    {\bf}                        %
    {.}                          %
    {.5em}                       %
    {}  %

\newcommand{\TV}{\mathrm{TV}}
\newcommand{\Bin}{\mathrm{Bin}}

\DeclareSymbolFont{rsfs}{U}{rsfs}{m}{n}
\DeclareSymbolFontAlphabet{\mathscrsfs}{rsfs}

\newcommand{\poly}{\mathrm{poly}}

\def\bI{{\boldsymbol I}}

\def\bL{{\boldsymbol L}}

\def\bP{{\boldsymbol P}}
\def\bQ{{\boldsymbol Q}}

\def\bV{{\boldsymbol V}}
\def\bW{{\boldsymbol W}}

\def\be{{\boldsymbol e}}

\def\bu{{\boldsymbol u}}
\def\bv{{\boldsymbol v}}
\def\bw{{\boldsymbol w}}
\def\bx{{\boldsymbol x}}

\def\btheta{{\boldsymbol \theta}}

\def\bPhi{{\boldsymbol \Phi}}

\newcommand{\VC}{\mathrm{VCdim}}
\newcommand{\VCpar}{\mathrm{VCdimPF}}

\def\cG{{\mathcal G}}

\def\cP{{\mathcal P}}

\def\cF{{\mathcal F}}

\def\cS{{\mathcal S}}
\def\cI{{\mathcal I}}

\def\cG{{\mathcal G}}

\def\cP{{\mathcal P}}

\def\cH{{\mathcal H}}
\def\cA{{\mathcal A}}

\def\NNet{{\sf NN}}

\def\cD{{\mathcal D}}
\def\cX{{\mathcal X}}
\def\cY{{\mathcal Y}}
\def\cF{{\mathcal F}}
\def\cS{{\mathcal S}}
\def\cI{{\mathcal I}}

\def\cX{{\mathcal X}}

\def\btheta{{\boldsymbol \theta}}

\def\bP{{\boldsymbol P}}

\def\bL{{\boldsymbol L}}

\def\bsigma{{\boldsymbol \sigma}}

\newcommand{\fNN}{f_{\NNet}}

\DeclarePairedDelimiter{\ceil}{\lceil}{\rceil}

\newenvironment{fminipage}%
  {\begin{Sbox}\begin{minipage}}%
  {\end{minipage}\end{Sbox}\fbox{\TheSbox}}

\newcommand{\NNkjuntas}{\cF_{\NNet, k\mbox{-}\mathrm{juntas}}}
\newcommand{\kjuntas}{\cF_{k\mbox{-}\mathrm{juntas}}}
\newcommand{\NNdectrees}{\cF_{\NNet, \mathrm{size}\mbox{-}s,\mathrm{depth}\mbox{-}r\mbox{ }\mathrm{trees}}}\newcommand{\NNdectreesLRH}{\cF^{\tau\mbox{-}\mathrm{LRH}}_{\NNet, \mathrm{size}\mbox{-}s,\mathrm{depth}\mbox{-}r\mbox{ }\mathrm{trees}}}
\newcommand{\dectrees}{\cF_{\mathrm{size}\mbox{-}s,\mathrm{depth}\mbox{-}r\mbox{ }\mathrm{trees}}}

\newcommand{\ANDfun}{\mathrm{AND}}

\newcommand{\algdepth}{R}

\newcommand{\dstinto}{\lesssim_{p,dst}}
\newcommand{\size}{\mathsf{size}}

\newcommand{\Adst}{\cA_{\mathrm{dist}}}
\newcommand{\Alrn}{\cA_{\mathrm{learn}}}
\newcommand{\Adstf}{\cA_{\mathrm{dist},f}}

\usepackage{hyperref}
\hypersetup{
    colorlinks,
    linkcolor={blue!80!black},
    citecolor={green!50!black},
}

\usepackage[top=1in, bottom=1in, left=1in, right=1in]{geometry}

\usepackage{amsmath}
\numberwithin{equation}{section}
\usepackage{amssymb}
\usepackage{mathtools}
\usepackage{amsthm}

\usepackage{thmtools}
\usepackage{thm-restate}

\usepackage[numbers, compress]{natbib}

\usepackage[capitalize,noabbrev]{cleveref}

\newtheorem{theorem}{Theorem}[section]
\newtheorem{proposition}[theorem]{Proposition}
\newtheorem{lemma}[theorem]{Lemma}

\newtheorem{claim}[theorem]{Claim}
\newtheorem{observation}[theorem]{Observation}
\newtheorem{definition}[theorem]{Definition}
\newtheorem{informaldefinition}[theorem]{Informal Definition}
\newtheorem{informaltheorem}[theorem]{Informal Theorem}

\theoremstyle{remark}
\newtheorem{remark}[theorem]{Remark}

\usepackage[textsize=tiny]{todonotes}

\usepackage{tikz}
\usetikzlibrary{shapes, arrows, calc, positioning, arrows.meta}

\title{Towards a theory of model distillation}

\author{Enric Boix-Adser\`a \\ \texttt{eboix@mit.edu}}

\begin{document}
\maketitle
\vspace{-2em}

\begin{abstract}

Distillation is the task of replacing a complicated machine learning model with a simpler model that approximates the original \cite{buciluǎ2006model,hinton2015distilling}. Despite many practical applications, basic questions about the extent to which models can be distilled, and the runtime and amount of data needed to distill, remain largely open.

To study these questions, we initiate a general theory of distillation, defining  PAC-distillation in an analogous way to PAC-learning \cite{valiant1984theory}. As applications of this theory: (1) we propose new algorithms to extract the knowledge stored in the trained weights of neural networks -- we show how to efficiently distill neural networks into succinct, explicit decision tree representations when possible by using the ``linear representation hypothesis''; and (2) we prove that distillation can be much cheaper than learning from scratch, and make progress on characterizing its complexity.
\end{abstract}

\setcounter{tocdepth}{2}
\tableofcontents
\clearpage

\section{Introduction}\label{sec:introduction}

This paper studies model distillation, which is the task of replacing a complicated machine learning model with a simpler one that approximates it well \cite{buciluǎ2006model,hinton2015distilling}.

Distillation has significant practical applications. First, distillation is an important tool for improving
computational efficiency: large models or ensembles of models can often be distilled to smaller models which are deployable at lower computational cost \cite{buciluǎ2006model,hinton2015distilling,polino2018model,frankle2018lottery,gou2021knowledge}. Second, a growing body of work uses distillation for 
interpretability: a ``black-box'' model such as a neural network can be studied by distilling it to the closest model from a more transparent class of models, such as linear functions or decision trees (see, e.g., \cite{arrieta2020explainable} and references within).

While distillation is widely used in practice, and sometimes succeeds in replacing large, complex models by smaller or simpler models with a minor loss in accuracy, a number of basic questions remain largely open:

\begin{enumerate}
\item \emph{To what extent can a given model be distilled, and on what properties of the model does this depend?}

\item 
\emph{How much data and runtime are needed to distill a given model?}
\end{enumerate}

\subsection{Contributions}
This paper makes progress on these questions. We summarize our contributions below:

\begin{itemize}
\item  \textbf{Formalization of distillation}: First, we formalize the distillation task via a new definition that we call \textit{PAC-distillation}. Our framework allows us to rigorously approach the above questions, letting us prove statements about the amount of data and runtime needed in order to distill a class of models into another class of models; see Section~\ref{sec:def-pac-dist}.

\item \textbf{New tools to distill neural networks}: Using this formalization, we propose novel algorithms to distill neural networks. Our algorithms exploit the structure in trained neural networks' weights to distill the models 
into more succinct and transparent forms. Most notably, we provide a provably efficient algorithm that distills a network that \textit{implicitly computes a decision tree} into an
\textit{explicitly-represented decision tree} which compactly represents the original network; see below for a schematic, and see Section~\ref{sec:case-study} for details.

\begin{figure}[h]
\centering
\includegraphics[trim={4cm 20cm 26cm 6cm},clip, scale=0.25]{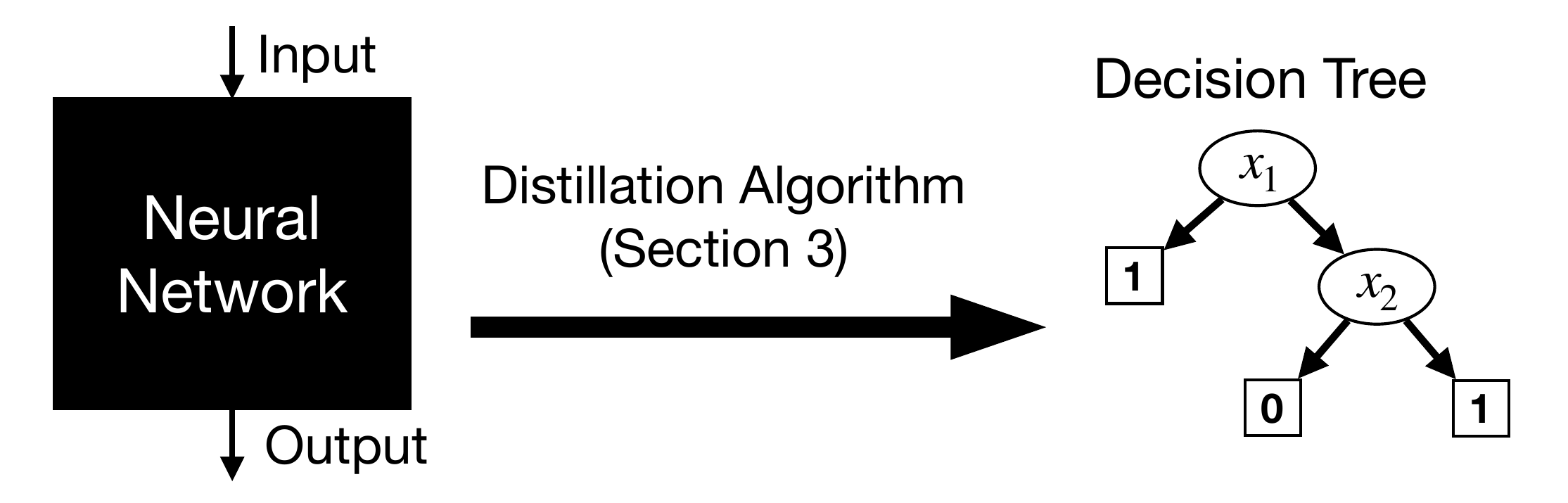}
\end{figure}

\item \textbf{General theory of distillation}: More generally, we initiate a computational and statistical theory of distillation that applies to arbitrary models beyond neural networks. A takeaway of our results is that in many cases distilling a trained model can be a much cheaper task in terms of data and compute than learning a model from scratch. This is an attractive property of distillation, since it means that distillation may be feasible even when few computing and data resources are available; see Sections~\ref{sec:computational} and \ref{sec:statistical}.

\item \textbf{Open problems}: Finally, we discuss extensions and new open directions; see Section~\ref{sec:extensions}.

\end{itemize}
\noindent Let us now delve into finer detail on each point.

\subsubsection{Defining PAC-distillation} We first introduce PAC-distillation, which formalizes the distillation problem from a class of models $\cF$ to a class of models $\cG$. This paper mostly focuses on the classification setting, where a model is a function $h : \cX \to \cY$ from datapoints to labels.
\begin{tcolorbox}[enhanced, frame hidden, sharp corners, boxsep=0pt, before skip=5pt, after skip=5pt, colback=blue!5!white]

\begin{informaldefinition}
Given a source model $f \in \cF$ and i.i.d. samples $x_1,\ldots,x_n$ from an unknown distribution $\cD$, 
{\em PAC-distillation} is the problem of finding a model $g \in \cG$ whose error approximating $f$ over the distribution is small:
$$\mathrm{error}_{f,\cD}(g) = \P_{x \sim \cD}[g(x) \neq f(x)]\,.$$

\end{informaldefinition}

\end{tcolorbox}

The full definition of PAC-distillation is in Section~\ref{sec:def-pac-dist} and variations, such as to the regression setting, are considered in Section~\ref{sec:extensions}.

\begin{remark}[Relation to PAC-learning] 
In analogy to PAC-learning \cite{valiant1984theory}, PAC stands for ``Probably Approximately Correct'', since the distillation process only has to be approximately correct with high probability. In the distillation setting we are given complete access to the source model $f \in \cF$ represented in some manner (e.g., if $\cF$ were a class of neural networks, $f$ would be given as the weights of the neural network). This contrasts with PAC-learning, where only the labels $f(x_1),\ldots,f(x_n)$ of the data samples would be known to the algorithm. 

Since a distillation algorithm gets extra information about the source function $f$ than just the sample labels, it is possible to PAC-distill from class $\cF$ to class $\cG$ whenever it is possible to PAC-learn concepts from class $\cF$ with hypotheses from class $\cG$. In this work, we will be primarily concerned with understanding when there is a separation -- i.e., when PAC-distillation requires fewer samples and/or runtime than PAC-learning.
\end{remark}

\begin{remark}[Relation to the KLCL model of \cite{ben2011learning}]
The closest setting studied in the literature is the Known-Labeling-Classifier-Learning (KLCL) model of \cite{ben2011learning}. The KLCL model corresponds to PAC-distillation, but in the special case where $\cF$ is the set of all possible labeling functions $\cX \to \cY$. In other words, in the KLCL model the labels are known, but they are not required to be instantiated by a classifier in a small class $\cF$. We discuss this highly-relevant setting further in Section~\ref{sec:related-work}.
\end{remark}

With the definition of PAC-distillation in hand, we now return to the questions raised above. Before describing our general results on fundamental computational and statistical limits for distillation, we focus on distilling neural networks, since these are of contemporary significance.

\subsubsection{New algorithmic tools to distill neural networks}

Neural networks are a highly expressive function class. As such, one should not expect that arbitrary neural networks can be usefully distilled into simpler models. The answer to whether a neural network can be distilled (question 1 above) therefore must hinge on \emph{some} structure being present in the weights of the trained neural network.

We present a novel distillation algorithm in this paper that makes use of a structure that may be present in a trained neural network: the \emph{linear representation property}. A network satisfies this property if important high-level features of the input can be expressed as a linear function of the trained neural network's internal representation. 
There is significant empirical evidence \cite{mikolov2013linguistic,alain2016understanding,vargas2020exploring,elhage2022toy,tigges2023linear,shi2016does,conneau2018you,allen2023physics,li2022emergent,nanda2023emergent,marks2023geometry,li2021implicit,dai2021knowledge,zhu2023physicsb,aspillaga2021inspecting,bolukbasi2016man,vargas2020exploring,ravfogel2022linear,ravfogel2023log,belrose2023leace,wang2023concept,hernandez2023inspecting,park2023linear,boix2022gulp} and growing theoretical evidence \cite{arora2016latent,abbe2022merged,damian2022neural,ba2022high,mousavi2022neural,abbe2023sgd,
dandi2023learning,bietti2023learning} supporting the hypothesis that trained neural networks satisfy the linear representation property; see Appendix~\ref{app:lrh-evidence} for a literature review.
\begin{tcolorbox}[enhanced, frame hidden, sharp corners, boxsep=0pt, before skip=5pt, after skip=5pt, colback=blue!5!white]
\begin{informaldefinition}[Linear representation property]
Let $\phi(x) \in \mathbb{R}^m$ be the neural network's {\em representation map}, which takes in an input $x$ and computes a vector embedding.\footnote{The representation map $\phi(x) \in \R^m$ is typically the concatenation of the functions computed at the $m$ internal neurons, which is what we use in our experiments. There are other choices, such as the gradient of the network with respect to its parameters, or the activations of the penultimate layer, that also make sense.}
\vspace{1em}

Let $\cG$ be a collection of real-valued functions of the input $x$. For any $\tau > 0$, the $\tau$-bounded linear representation property with respect to $\cG$ states that for each function $g \in \cG$, there is $w_g \in \R^m$ such that $\|w_g\| \leq \tau$ and for all inputs $x$ we have
\begin{align*}
w_g^{\top} \phi(x) = g(x)\,.
\end{align*}
\end{informaldefinition}
\end{tcolorbox}

In this paper we prove that under the linear representation property a neural network that \textit{implicitly computes a decision tree} can be distilled into an \emph{explicit and succinct} decision tree. The set of functions $\cG$ that we assume that the trained network linearly represents is the set of the intermediate computations of the decision tree, consisting of all of AND statements computed by paths starting from the root and going to any internal node or leaf. Our distillation algorithm is 
efficient (polynomial time). 

\begin{tcolorbox}[enhanced, frame hidden, sharp corners, boxsep=0pt, before skip=5pt, after skip=5pt, colback=blue!5!white]\begin{informaltheorem}[The linear representation property can be used to efficiently distill neural networks into decision trees]
Suppose that we are given:
\begin{itemize}
    \item a neural network $\fNN(x)$ that implicitly computes a decision tree of size $s$ and depth $r$
    \item the network's representation $\phi(x) \in \R^m$, which satisfies the $\tau$-bounded linear representation property with respect to the internal computations of the decision tree
\end{itemize}
Then, if the input distribution is uniform over $\{0,1\}^d$, the network can be distilled to an equivalent, explicitly represented tree of size $s$ in $\mathrm{poly}(d,s,\tau,\max_x \|\phi(x)\|,m)$ time, where each network and representation evaluation takes unit time.

If the input distribution $\cD$ is arbitrary over $\{0,1\}^d$, then distillation is possible in $\mathrm{poly}(d,2^r,\tau,\max_x \|\phi(x)\|, m)$ time and samples from $\cD$.
\end{informaltheorem}
\end{tcolorbox}

\begin{remark}[Comparison to query learning, and learning from random examples]
This constitutes the first polynomial-time algorithm for distillation of neural networks into decision trees. The prior best algorithm for this task uses query access to the neural network model and runs in time $\tilde{O}(d^2) \cdot s^{O(\log\log s)}$ when the input distribution is assumed to be uniform over the binary hypercube \cite{blanc2022properly}.
The caveat in our result is that we additionally require the linear representation property. Nevertheless, we verify experimentally that this property can hold in practice. In Section~\ref{sec:empirical},  we implement our algorithm to successfully extract decision trees from deep neural networks that have been trained to learn a decision tree. 

Our distillation result should also be contrasted with the current best theoretical guarantee of $d^{O(\log(s))}$ time and samples for learning decision trees from random examples \cite{ehrenfeucht1989learning}, and which is not guaranteed to return a $s$-size tree but can potentially return a tree of $d^{\Omega(\log(s))}$ size. Thus, this is a natural setting where distillation can potentially be much faster than learning from random examples.
\end{remark}

\begin{remark}[Beyond decision trees] Of course, neural networks can and do learn functions that are \emph{not} given by small decision trees, so distillation into small decision trees may not be possible. We believe that under the linear representation property it may be possible to efficiently distill neural networks into other simple model classes that are more expressive than decision trees but better capture the internal computations of the network. Indeed, we view leveraging of the linear representation property as the more significant conceptual contribution of this theorem, rather than the specific polynomial-time algorithm for decision tree distillation.
In ongoing work, we are currently exploring distillation to other classes beyond decision trees, as mentioned in Section~\ref{sec:extensions}.

\end{remark}

The above result addresses the questions 1 and 2, by giving specific, nontrivial conditions under which models with hidden structure can be efficiently distilled (and more broadly our result suggests that the linear representation property may be useful for distilling networks beyond the setting of decision trees). We now turn to providing theoretical principles on distillation that address questions 1 and 2 in more generality.

\subsubsection{General computational and statistical theory of distillation} We initiate a general computational and statistical theory of distillation, analogous to the theory of PAC-learning. First, we observe that distillation can be used to define a web of computational reductions. Distillation creates a partial order between model classes, capturing the \emph{computational feasibility} of converting from one model class to another. We are optimistic that this framework can lead to a computational theory of interpretability, formalizing what it means for a given model class to be uniformly more interpretable than another. This is discussed in Section~\ref{sec:computational}.

Towards deriving fundamental statistical limits, 
we also explore the sample complexity of distillation. 
One of our main results is on \textit{perfect distillation},  
which is the setting where for every model in $\cF$ it is possible to distill arbitrarily well to $\cG$ given sufficiently many samples. In this setting we characterize the sample complexity of distillation for finite and countable model classes, and we show that very few samples are needed to distill.
\begin{tcolorbox}[enhanced, frame hidden, sharp corners, boxsep=0pt, before skip=5pt, after skip=5pt, colback=blue!5!white]
\begin{informaltheorem}[Whenever it is possible to distill to arbitrary accuracy, very few samples are needed]\label{informalthm:perfect-distillation}
Let $\cF$ and $\cG$ be model classes of countable size. Whenever it is possible to ``perfectly distill'' from $\cF$ into $\cG$ (see Definition~\ref{def:perfectly-distillable}), it is possible to $\epsilon$-approximately distill with $\delta$ probability of error using only
$O(\log(1/\delta)/\epsilon)$ samples. Furthermore, this bound is tight.
\end{informaltheorem}
\end{tcolorbox}
The conceptual takeaway is that in general far fewer samples are needed to distill from class $\cF$ to class $\cG$ than are needed to learn the model in class $\cG$ from scratch. The number of samples needed to distill here is always bounded independently of $\cF$ and $\cG$, even in cases when $\cG$ is not learnable. This supports a machine learning paradigm in which a large 
institution or company trains models using a huge amount of data, and then smaller entities can distill the trained model using far 
fewer resources. Of course, Informal Theorem~\ref{informalthm:perfect-distillation} does not provide a computationally efficient algorithm, and tradeoffs between the amount of computation and data used to distill are highly interesting to study in future work.

Finally, we consider agnostic distillation, where the goal is to distill to a model that competes with the best model in $\cG$. We provide lower bounds and upper bounds on the sample complexity. Fan Chen and Sasha Rakhlin have pointed out in private communication \cite{chenrakhlin} that the sample complexity of agnostic distillation is equivalent to the sample complexity in the ``General Learning Setting'' of \cite{vapnik1995nature}, and therefore it does not admit a characterization via a VC-dimension-like quantity \cite{ben2017learning}.

\subsection{Related work}\label{sec:related-work}

\paragraph{Distillation applied to efficient inference} 
Model distillation was proposed in \cite{buciluǎ2006model} under the name ``model compression''. The setting in that work was finding a single small model that matched the performance of a large ensemble of models. This was later popularized by \cite{hinton2015distilling}, which proposed the method of training a small neural network model with a loss function that drove it to match the logits of the larger ensemble model. Extensive empirical work has shown that this method can often distill large neural network models to smaller ones with minimal loss in accuracy (see e.g., the surveys \cite{gou2021knowledge,xu2024survey} and references within). There are alternative methods for distilling large models to smaller ones that also have empirical success, such as methods based on pruning \cite{zhu2017prune,frankle2018lottery,michel2019sixteen,wang2019structured,blalock2020state}, methods based on replacing individual layers with lightweight approximations (e.g., \cite{massaroli2023laughing,sharma2023truth}), and methods based on quantizing the weights of the models \cite{gupta2015deep,gholami2022survey}.

\paragraph{Distillation applied to interpretability} Distillation is also an important tool in the model interpretability literature, where it is sometimes known as ``model simplification'' (see e.g., the survey \cite{arrieta2020explainable} and references within). This approach to interpretability seeks to extract rules from complicated models such as neural networks by distilling them, if possible, into small logical formulas \cite{mcmillan1991rule,mcmillan1992rule,thrun1993extracting,andrews1995survey}, decision trees \cite{craven1994using,craven1995extracting,breiman1996born,van2007seeing,johansson2011one,zhou2016interpreting,bastani2017interpretability,vandewiele2017genetic,frosst2017distilling}, generalized additive models \cite{liang2008structure,ribeiro2016should,tan2018distill}, or wavelets \cite{ha2021adaptive}. Those more transparent models can then be understood or interpreted using downstream techniques (such as model purification \cite{lengerich2020purifying} in the case of generalized additive models), yielding insight into the original model.

\textit{Other approaches to interpretability}. It should be noted that there are many other approaches to interpreting black-box models outside of distillation, such as methods based on feature visualization and attribution (see, e.g., \cite{murdoch2019interpretable,arrieta2020explainable} and references within). Another angle is to avoid training black-box models and instead train models that are interpretable by design
(see, e.g., \cite{rudin2022interpretable,chattopadhyay2022interpretable}).

\paragraph{Mathematical analyses of distillation}
As previously discussed, the most related work is \cite{ben2011learning}, which formulates the Known-Labeling-Classifier-Learning (KLCL) model. Under this model, the labels are known, but the distribution over the inputs is not. The goal is to find a model that competes with the best possible model in a class $\cG$. This corresponds to agnostic PAC-distillation, where the source class $\cF$ is the set of all possible labeling functions. \cite{ben2011learning} proves a trichotomy in terms of the dependence on $\epsilon$: either $0$ samples, $O(1/\epsilon)$, or $\Omega(1/\epsilon^2)$ samples are needed, and these three regimes correspond to combinatorial properties of the class $\cG$. Furthermore, \cite{ben2011learning} proves that there are classes for which learning in the KLCL model is impossible with any finite number of samples. This implies that there are settings in which agnostic PAC-distillation is impossible with any finite number of samples.

The work of 
\cite{urner2011access} studies the related setting of semi-supervised learning, where there is a small amount of labeled data and a large amount of unlabeled data. That 
paper gives general sufficient conditions under which a two-stage algorithm involving learning and then distilling can reduce the amount of labelled data needed to learn, and it explores these conditions in the context of learning halfspaces and learning with nearest neighbors.

Another pair of highly-relevant works are \cite{zhou2018approximation} and \cite{zhou2022generic}, which study the stability of distilling models, showing that in practice a significant amount of data is needed to ensure that the same model is consistently outputted by distillation across various trials with independent samples. Of these papers, \cite{zhou2022generic} proposes a general procedure for distillation that treats the process as a multiple-hypothesis testing problem between student hypotheses, and provides upper bounds on the number of samples needed to distill with this procedure, based on certain measures of complexity of the model classes.

In the domain of neural networks, several hypotheses for the practical effectiveness of distillation have been proposed. In \cite{allen2020towards} it is argued that distillation succeeds because of ``multi-view'' structure in data that is learned by an ensemble of models, and is captured by the distilled model. In \cite{pham2022revisiting} it is argued that distillation prefers flatter minima in the loss landscape, which can even enable the distilled model to generalize better than the original. The paper \cite{phuong2019towards} analyzes the distillation of linear and deep linear models by optimizing the cross-entropy loss on the soft labels of the teacher model. There, it is proved that distillation with a small amount of data can obtain much lower risk than learning from the same amount of data, because of the interaction between the optimizer and the geometry of the data.

\subsection{Notation}
Let $[m] = \{1,\ldots,m\}$ denote the set of $m$ elements. Given a vector $v \in \R^m$, and a subset $S \subseteq [m]$ of indices, let $v_S \in \R^S$ denote the subvector indexed by those indices. Let $\mathrm{poly}(a_1,a_2,\ldots,a_k)$ denote a polynomially-bounded function of $a_1,\ldots,a_k$.

\section{PAC-distillation}\label{sec:def-pac-dist}
We formally define \textit{PAC-distillation}, which is the problem of translating a function from source class $\cF$ to an approximately equivalent function in target class $\cG$. The acronym PAC stands for ``Probably Approximately Correct'': the distillation must be $\epsilon$-approximately correct with probability at least $1 - \delta$ for some parameters $\epsilon,\delta$.
    
\begin{figure}
\centering

\begin{tabular}{c|c|c}
& \begin{tabular}{c}\textbf{PAC-learning \cite{valiant1984theory}} \\ concept class $\cF$, hypothesis class $\cG$\end{tabular} & \begin{tabular}{c}\textbf{PAC-distillation [{\color{green!50!black}this work}]} \\ source class $\cF$, target class $\cG$\end{tabular} \\
\hline
Unknown: & \begin{tabular}{c} distribution $\cD$ \\
concept $f \in \cF$\end{tabular} & distribution $\cD$ \\
\hline
Known: & \begin{tabular}{c} 
examples $x_1,\ldots,x_n \stackrel{i.i.d.}{\sim} \cD$ \\ labels $f(x_1),\ldots,f(x_n)$ \end{tabular}& \begin{tabular}{c} 
examples $x_1,\ldots,x_n \stackrel{i.i.d.}{\sim} \cD$ \\ model $f \in \cF$ \end{tabular} \\
\hline
Want: & hypothesis $g \in \cG$ minimizing $\mathrm{error}_{f,\cD}(g)$& model $g \in \cG$ minimizing $\mathrm{error}_{f,\cD}(g)$
\end{tabular}
\caption{Summary of PAC-learning vs. PAC-distillation as defined in Section~\ref{sec:def-pac-dist}. The difference is that in PAC-distillation, the source model $f$ is one of the inputs to the distillation algorithm. For example, if the goal is to distill a neural network, then the neural network's weights are inputted to the distillation algorithm.
}
\end{figure}
    
    \subsection{PAC-learning} Our definition is based on \textit{PAC-learning} \cite{valiant1984theory}, a foundational definition in computational learning theory \cite{kearns1994introduction,mohri2018foundations}. In the PAC-learning formalism, there is:
    \begin{itemize}
        \item an \textit{input space} $\cX$, {\color{brown} \\ (e.g., the set of all possible images)}
        \item a \textit{label space} $\cY$, \\ {\color{brown} (e.g., the set of possible classification labels of those images)}
        \item a \textit{distribution} $\cD$ over the input space $\cX$, \\ {\color{brown} (e.g., the distribution of natural images on the Internet)}
        \item a \textit{concept class} $\cF$, which is a set of functions from $\cX$ to $\cY$, \\ {\color{brown} (e.g., the set of possible ground-truth classifications)}
        \item and a \textit{hypothesis class} $\cG$, which is also a set of functions from $\cX$ to $\cY$. \\ {\color{brown} (e.g., the set of  neural network classifiers)}
    \end{itemize}

The goal of a learning algorithm is to output an approximation $g \in \cG$ of a ground-truth function $f \in \cF$, with respect to the distribution $\cD$. A learning algorithm seeks to minimize the probability that the returned hypothesis makes a mistake on a fresh test data point:
\begin{align}\label{eq:error-of-hypothesis}
\mathrm{error}(g) = \mathrm{error}_{f,\cD}(g) = \P_{x \sim \cD}[g(x) \neq f(x)]\,.
\end{align}
A major difficulty is that the learning algorithm does not have direct access to (i) the ground-truth function $f$, or (ii) the distribution $\cD$. The inputs are a set of random examples $S = (x_1,\ldots,x_n) \in \cX^n$ drawn i.i.d. from $\cD$, and corresponding ground-truth labels $L = (f(x_1),\ldots,f(x_n)) \in \cY^n$. The learner outputs a hypothesis:
$$\Alrn(S,L) \in \cG.$$ 
    
A concept class $\cF$ is efficiently PAC-learnable with hypothesis class $\cG$ if there is a learner that outputs a low-error hypothesis using a small number of samples and small amount of time.\footnote{Typically, time and samples are required to scale polynomially in the input and classifier size for PAC-learnability. The definition here is more convenient to discuss fine-grained runtime and sample complexity bounds as in this paper.}
\begin{tcolorbox}[enhanced, frame hidden, sharp corners, boxsep=0pt, before skip=5pt, after skip=5pt, colback=blue!5!white]
    \begin{definition}[PAC-learning; \cite{valiant1984theory}]
        Concept class $\cF$ is $(\epsilon,\delta)$-learnable by algorithm $\Alrn$ in $n$ samples and $t$ time if for any distribution $\cD$ on $\cX$, and any concept $f \in \cF$, the algorithm runs in $\leq t$ time and
        \begin{align*}
            \P_{S \sim \cD^n}[\mathrm{error}_{f,\cD}(\Alrn(S,L)) \leq \epsilon] \geq 1- \delta\,.
        \end{align*}
    \end{definition}
    \end{tcolorbox}

    \subsection{PAC-distillation} 
    We adapt the definition of PAC-learnability to the distillation setting, where we refer to $\cF$ as the \textit{source class}, and to $\cG$ as the \textit{target class}. The difference is that the distillation algorithm $\Adst$ is also given the source function $f \in \cF$ as one of its inputs:
    $$\Adst(S,f) \in \cG\,.$$

\begin{tcolorbox}[enhanced, frame hidden, sharp corners, boxsep=0pt, before skip=5pt, after skip=5pt, colback=blue!5!white]
    \begin{definition}[PAC-distillation] Source class $\cF$ is $(\epsilon,\delta)$-distillable into target class $\cG$ by algorithm $\Adst$ in $n$ samples and $t$ time if for any distribution $\cD$ on $\cX$, and any source model $f \in \cF$, the algorithm runs in $\leq t$ time and:
        \begin{align*}
            \P_{S \sim \cD^n}[\mathrm{error}_{f,\cD}(\Adst(S,f)) \leq \epsilon] \geq 1- \delta\,.
        \end{align*}

    \end{definition}
    \end{tcolorbox}

PAC-distillation is an easier problem than PAC-learning because given the samples $S = (x_1,\ldots,x_n)$ and the function $f \in \cF$, the labels $L = (f(x_1),\ldots,f(x_n))$ can be computed, so a learning algorithm can be applied. Nevertheless, distillation remains challenging from the statistical angle because the distribution $\cD$ is unknown, and also from the computational angle because one must convert between the representation of functions in function class $\cF$ and those in function class $\cG$.

If the source class $\cF$ and the target class $\cG$ are equal (i.e., $\cF = \cG$), as is often the case for learning, then the distillation problem is trivial -- the algorithm can simply take in $f \in \cF$ and output $f \in \cG$. The distillation problem is challenging and interesting when $\cF$ and $\cG$ are distinct function classes.

\subsubsection{Agnostic PAC-distillation}

We also define an \textit{agnostic} variant of PAC-distillation. The goal is to find $g \in \cG$ that competes with the best possible target function in $\cG$ instead of having small absolute error.
\begin{tcolorbox}[enhanced, frame hidden, sharp corners, boxsep=0pt, before skip=5pt, after skip=5pt, colback=blue!5!white]
\begin{definition}[Agnostic PAC-distillation]
Source class $\cF$ is $(\epsilon,\delta)$-agnostically distillable into target class $\cG$ by algorithm $\Adst$ in $n$ samples and $t$ time if for any distribution $\cD$ on $\cX$, and any source model $f \in \cF$, the algorithm runs in $\leq t$ time and:
\begin{align*}
\P_{S \sim \cD^n}[\mathrm{error}_{f,\cD}(\cA_{dist}(S,f)) - \min_{g \in \cG}\mathrm{error}_{f,\cD}(g) \leq \epsilon] \geq 1 - \delta\,.
\end{align*}
\end{definition}
\end{tcolorbox}
The term \textit{agnostic} is inherited from \textit{agnostic PAC-learning}. See Section~\ref{sec:extensions} for further extensions.

\section{Two case studies in distilling neural networks}\label{sec:case-study}

Since neural networks are the most common class of ``black-box functions''  in practice, we motivate PAC-distillation with two case studies showing how to distill neural networks that have learned simple logical functions. Our algorithms distill the neural networks into more interpretable forms in time faster than required to learn from scratch:
\begin{itemize}
\item In Section~\ref{ssec:juntas}, we warm up with the setting when the neural network has implicitly learned a \textit{junta}, i.e., a function of the input that depends only on a small subset of the features (formal definition below). Here, we use previous results on query learning \cite{bshouty2016exact} to efficiently distill the model into an \textit{explicit representation} of the junta.
\item In Section~\ref{ssec:decisiontrees}, we consider the more general setting when the neural network has implicitly learned a \textit{decision tree}. We provide a new result, showing that under the ``linear representation hypothesis'' (see, e.g., \cite{mikolov2013linguistic,elhage2022toy} and further references in Appendix~\ref{app:lrh-evidence}) we can distill the neural network into a decision tree in polynomial time.

\end{itemize}

\subsection{Warm-up: provably-efficient distillation of networks into juntas}\label{ssec:juntas}

We warm up with a previously-known result on learning juntas with queries \cite{kushilevitz1991learning,damaschke1998adaptive,damaschke1998computational,de2005optimal,bshouty2016exact}; this will set the stage for our new result in Sections~\ref{ssec:decisiontrees}.

\begin{tcolorbox}[enhanced, frame hidden, sharp corners, boxsep=0pt, before skip=5pt, after skip=5pt, colback=blue!5!white]
\begin{definition}[Junta]\label{def:junta}
A $k$-junta is a function $f : \{0,1\}^d \to \{0,1\}$ that depends only on a subset $S \subseteq [d]$ of the coordinates, of size $|S| = k$. Formally, $f(\bx) = h(\bx_S)$ for some $h : \{0,1\}^k \to \{0,1\}$.
\end{definition}
\end{tcolorbox}

\paragraph{Learning juntas from random examples is hard}  Learning $k$-juntas from noisy data is widely held to be a computationally hard problem: it is conjectured to require at least $\Omega(d^k)$ operations as the dimension $d$ grows  \cite{blum1994weakly}. This hardness result has been proved rigorously under restricted models of computation such as the Statistical Query (SQ) framework \cite{blum1994weakly}.

\paragraph{Can we distill networks into juntas efficiently?} Is it possible to distill a neural network into a junta in significantly fewer operations than learning the junta from scratch, avoiding the $\Omega(d^k)$ complexity barrier, and instead running in $d^{O(1)}$ time, where the constant in the exponent does not grow with $k$?

Formally, suppose that we are given a vector of parameters $\theta$ that specify a neural network $f_{\theta} : \{0,1\}^d \to \R$ that implicitly computes a $k$-junta. We seek to extract an explicit representation of the junta in terms of the subset $S \subseteq [d]$ of the coordinates and a truth table of the function on those coordinates. In other words, we seek to distill from the class
\begin{align*}
\NNkjuntas = \{\mbox{neural networks } f_{\theta} \mbox{ which implicitly compute a } k\mbox{-junta}\}
\end{align*}
into the class of $k$-juntas
\begin{align*}
\kjuntas = \{k\mbox{-juntas}, \mbox{ explicitly represented by } S \mbox{ and } h : \{0,1\}^k \to \{0,1\}\}\,.
\end{align*}
This is a nontrivial computational problem, since the neural networks in $\NNkjuntas$ are represented by the vector of weights $\theta$, whereas the functions in $\cF_{k\mbox{-juntas}}$ are represented by the subset of the coordinates on which they depend, and a $2^k$-size truth table on those coordinates.

We have been intentionally vague about the neural network's architecture, because it turns out that the architecture does not matter. Using only evaluation access to the network, it is possible to distill the network into a junta very 
efficiently. This follows from a previously-known result on learning with queries:

\begin{tcolorbox}[enhanced, frame hidden, sharp corners, boxsep=0pt, before skip=5pt, after skip=5pt, colback=blue!5!white]
\begin{proposition}[Efficient distillation of networks into juntas; cf. \cite{bshouty2016exact}]\label{prop:distilling-with-juntas} Suppose that each network evaluation takes $T$ time. Then, for any $0 < \delta < 1$, we can $(\epsilon=0,\delta)$-distill $\NNkjuntas$ into $\kjuntas$ in $0$ samples and $O(T \cdot k \log d + T \cdot 2^k \log\frac{1}{\delta}) +\poly(2^k,d)$ time.
\end{proposition}
\end{tcolorbox}
\begin{proof}
Directly implied by Theorem~13 of \cite{bshouty2016exact} on the query complexity of learning juntas. Each query can be simulated by one evaluation of the neural network.
\end{proof}

The key points are that: (a) learning the $k$-junta from noisy data takes at least $\Omega(d^k)$ time under standard hardness assumptions \cite{blum1994weakly}; (b) we can distill an explicit representation of the junta from the trained neural network much faster, in $\poly(d,2^k)$ network evaluations and time, where the exponent on $d$ does not scale in $k$.

\paragraph{Significance of this result}
The distillation phase has vanishing computational cost compared to the learning phase. This supports a ``\textit{Learn first, distill later}'' paradigm of data science, where the bulk of the compute is first dedicated to training a model that fits the data using a flexible and general-purpose learning algorithm such as training a neural network\footnote{In fact, the model training for learning juntas runs in the conjectured optimal run-time \cite{barak2022hidden,abbe2023sgd,edelman2023pareto}, and the runtime adapts to hierarchical structure in the data \cite{abbe2021staircase,abbe2022merged,abbe2022non,abbe2023sgd}.}, and then the knowledge in this network can be efficiently extracted at much lower cost.

\subsection{Provably-efficient distillation of networks into decision trees}\label{ssec:decisiontrees}

Next, we consider the more general setting of distilling neural networks into decision trees, which are popular classifiers that are often touted for their transparency \cite{breiman1984classification,rudin2022interpretable}.

We provide a novel algorithm for distilling neural networks into decision trees in polynomial time, which is faster than the current fastest algorithms for learning decision trees from data or queries \cite{ehrenfeucht1989learning,blanc2022properly}, but our algorithm uses the neural network's structure beyond evaluation access to the network.

Recall the definition of a decision tree; see Figure~\ref{fig:dec-tree} for an example.

\begin{tcolorbox}[enhanced, frame hidden, sharp corners, boxsep=0pt, before skip=5pt, after skip=5pt, colback=blue!5!white]
\begin{definition}[Decision tree]\label{def:dectrees} A decision tree $T$ is a binary tree, where each internal node is labeled by a variable in $\{x_1,\ldots,x_d\}$, and each leaf is labeled by the output $0$ or $1$. In a notational overlaod, the tree represents a function $T : \{0,1\}^d \to \{0,1\}$ where $T(\bx)$ is computed by starting at the root node, and following the path to the leaves determined by $\bx$.
\end{definition}
\end{tcolorbox}

\begin{figure}
\centering
\includegraphics[scale=0.5]{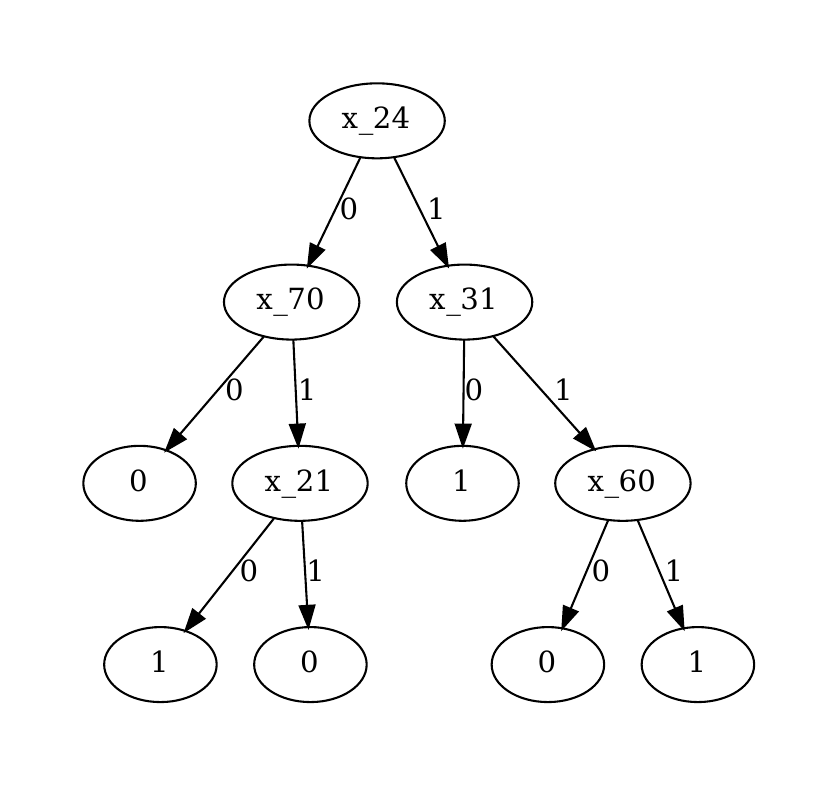}
\caption{Example of a depth-3 decision tree with 11 nodes and $d = 100$.}\label{fig:dec-tree}
\end{figure}

\paragraph{Learning decision trees from random examples is hard} The problem of finding the smallest decision tree from random examples that fits the data exactly or approximately is NP-hard \cite{angluin1976remarks,pitt1988computational,koch2023superpolynomial,koch2023properly}, which suggests that no polynomial time algorithm may be possible. The fastest approximation algorithm for learning decision trees from data runs in $d^{O(\log(s))}$-time and returns a tree that may be of size $d^{\Omega(\log(s))}$ \cite{ehrenfeucht1989learning}, assuming is a ground-truth size-$s$ tree that fits the data perfectly. This algorithm matches corresponding $d^{\Omega(\log(s))}$-complexity lower bounds for learning in the Statistical Query (SQ) model \cite{blum1994weakly}, indicating it might be optimal.

\paragraph{Learning decision trees from queries is also hard} Since learning decision trees in NP-hard, we can ask analogously to the previous section: is it easier to distill a model into a decision tree, than to learn the decision tree from scratch? Formally, is it possible to distill the class
$$\NNdectrees = \{\mbox{neural networks } f_{\theta} \mbox{ which implicitly compute decision trees}\}$$ into the class $$\dectrees = \{\mbox{decision trees, explicitly represented}\}\,,$$
in polynomial time and number of samples?

Unfortunately, if we only use evaluation access to the network this problem is still NP-hard \cite{koch2023properly}. And even under the extra assumption that the data distribution $\cD$ is uniform on $\{0,1\}^d$, the current best algorithm runs in $\tilde{O}(d^2) \cdot (s/\epsilon)^{O(\log((\log s)/\epsilon))}$ time \cite{blanc2022properly}. This is better than the naive $d^{O(\log(s))}$ algorithm learning from random samples, but does not yet achieve polynomial time.

\paragraph{We will use additional neural network structure to distill in polynomial time}

We sidestep the limitations of previous decision tree learning algorithms by using extra structure from the neural network, \textit{beyond evaluation access to the network}. This will ultimately allow us to distill networks into decision trees in polynomial time.\footnote{Our algorithm runs $\poly(d,1/\epsilon,s)$ time when the input distribution is uniform, and $\poly(d,1/\epsilon,s,2^r)$ time for arbitrary non-uniform distributions -- see Theorem~\ref{thm:decision-tree-pac-dst}. For arbitrary non-uniform distributions and $r = O(\log(s))$, this is $\poly(d,1/\epsilon,s)$ time, whereas previously known algorithms require $d^{\Omega(\log(s))}$-time although they do not require the Linear Representation Hypothesis. }

We must first explain the structure of the network that our distillation algorithm uses. We use the vector-valued representation map $\varphi_{\theta} : \cX \to \R^m$ associated with the neural network. The representation $\varphi_{\theta}(\bx) \in \R^m$ is the concatenation of all of the internal neurons' activations under input $\bx$.\footnote{Sometimes $\varphi_{\theta}$ is instead defined as the penultimate layer of activations, or the gradient of the network with respect to the weights.} Surprisingly, this representation $\varphi_{\theta}(\bx)$ can often represent useful high-level features of the data via linear regression. This finding was called the ``linear representation hypothesis'' (LRH) by \cite{elhage2022toy}. Formally:
\begin{tcolorbox}[enhanced, frame hidden, sharp corners, boxsep=0pt, before skip=5pt, after skip=5pt, colback=blue!5!white]
\begin{restatable}
[Linear representation hypothesis]{definition}{lrhdef}
Let $\cG$ be a collection of functions $g : \cX \to \R$ for each $g \in \cG$. For any $\tau > 0$, the $\tau$-LRH with respect to $\cG$ states that for all $g \in \cG$ there is a coefficient vector $\bw_g \in \R^m$ such that $\|\bw_g\| \leq \tau$ and
\begin{align*}
\bw_g \cdot \varphi(\bx) = g(\bx) \quad \mbox{ for all } \bx \in \cX\,.
\end{align*}
\end{restatable}
\end{tcolorbox}

The parameter $\tau$ is a bound on the norm of the coefficients that controls the complexity of the linear combination of the features in $\varphi_{\theta}$. A smaller $\tau$ means that the network can more easily represent the functions in $\cG$. The set $\cG$ varies depending on the setting, but should be thought of as a collection of ``high-level features'' that could include important intermediate steps in the computation of the neural network.

There is rapidly mounting empirical \cite{mikolov2013linguistic,tigges2023linear,shi2016does,conneau2018you,allen2023physics,li2022emergent,nanda2023emergent,marks2023geometry,li2021implicit,dai2021knowledge,zhu2023physicsb,fellbaum1998wordnet,aspillaga2021inspecting,bolukbasi2016man,vargas2020exploring,ravfogel2022linear,ravfogel2023log,belrose2023leace,wang2023concept,hernandez2023inspecting,park2023linear,boix2022gulp} and theoretical \cite{arora2016latent,abbe2022merged,damian2022neural,ba2022high,mousavi2022neural,abbe2023sgd,
dandi2023learning,bietti2023learning} evidence across several domains that trained neural networks satisfy the LRH, for appropriate sets of features $\cG$. To keep the presentation focused on decision tree distillation, we discuss evidence for the LRH in Appendix~\ref{app:lrh-evidence}.

\paragraph{The LRH for decision trees} In decision trees, the intermediate computations of the decision tree correspond to all paths of the tree that start at the root. We postulate that these intermediate computations are linearly encoded by the neural network's learned representation. This is defined formally below.

\begin{tcolorbox}[enhanced, frame hidden, sharp corners, boxsep=0pt, before skip=5pt, after skip=5pt, colback=blue!5!white]
\begin{definition} A {\em clause} is a set of literals $S \subseteq \{x_1,\ldots,x_d,\neg x_1,\dots,\neg x_d\}$. The function $\ANDfun_S(\bx)$ is defined to be 1 if all literals in $S$ are true, and $0$ otherwise. We say that $S$ is a nondegenerate $k$-clause if $|S| = k$ and each variable appears at most once in $S$ (i.e., for all $i$, we have $|S \cap \{x_i,\neg x_i\}| \leq 1$).
\end{definition}
\end{tcolorbox}
Given a decision tree $T$, the set of intermediate computations is:
\begin{align*}
\cG_T = \{\ANDfun_{S}(\bx) \mid \mbox{clauses } S \mbox{ formed by a path starting at the root of } T\}
\end{align*}
See Figure~\ref{fig:dec-tree-inter-comp} for an example.

\begin{figure}[h]
\centering
\begin{tabular}{cc}
\raisebox{-.5\height}{\includegraphics[scale=0.5]{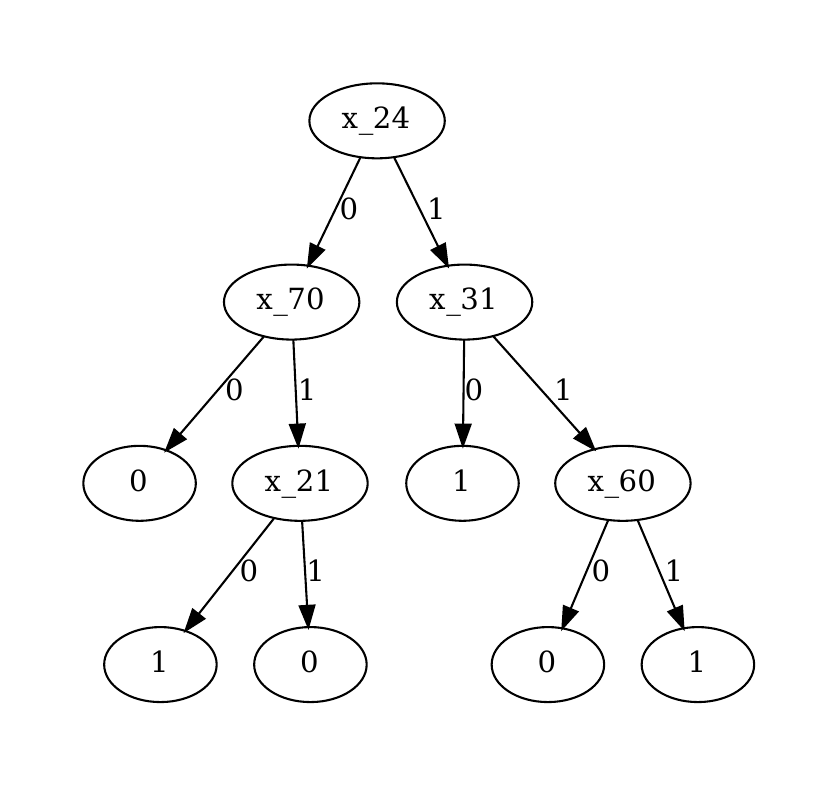}} & \fbox{\begin{tabular}{c}Intermediate computations $\cG_T$: \\
$1$ \\
$\neg x_{24}$ \\
$x_{24}$ \\
$\neg x_{24} \wedge \neg x_{70}$ \\
$\neg x_{24} \wedge x_{70}$ \\
$x_{24} \wedge \neg x_{31}$ \\
$x_{24} \wedge x_{31}$ \\
$\neg x_{24} \wedge x_{70} \wedge x_{21}$ \\

$\neg x_{24} \wedge x_{70} \wedge x_{21}$ \\

$x_{24} \wedge x_{31} \wedge \neg x_{60}$ \\

$x_{24} \wedge x_{31} \wedge x_{60}$
\end{tabular}} \\
(a) & (b)
\end{tabular}
\caption{(a) Example of a depth-3 decision tree with 11 nodes and $d = 100$. (b) List of all intermediate computations in this tree; there is one $\ANDfun$ function per path starting at the root. Note that these paths do not have to end at the leaves. For convenience, we also include the path of length zero, which is why $\ANDfun_{\emptyset} \equiv 1$ is also one of the intermediate computations.}\label{fig:dec-tree-inter-comp}
\end{figure}
Networks that have learned decision trees and satisfy the LRH belong to the set
\begin{align*}
\NNdectreesLRH = \{&\mbox{neural networks } f_{\theta} \mbox{ which implicitly compute a decision tree } T \\
&\mbox{ such that the network's representation } \varphi_{\theta} \mbox{ satisfies } \tau\mbox{-LRH for features } \cG_T\}
\end{align*}

To make sure that these definitions are clear, we restate them in slightly different form. Let $T$ be a decision tree, and let $S \subseteq \{x_1,\ldots,x_d,\neg x_1,\ldots,\neg x_d\}$ be a clause formed by taking a path in the tree starting from the root. Then, a network $f_{\theta} \in \NNdectreesLRH$ which has implicitly learned $T$ will have a representation $\varphi_{\theta}$ such that there is a vector of coefficients $\bw_S \in \R^m$ of bounded norm $\|\bw_S\| \leq \tau$ such that
\begin{align*}
\bw_S \cdot \varphi_{\theta}(\bx) = \ANDfun_S(\bx) \mbox{ for all } \bx\,.
\end{align*}

We are now ready to present our main result in this subsection, which is a polynomial time algorithm for distilling networks into decision trees, under the LRH condition. Each evaluation of the network or the network's representation is assumed to take unit time.
\begin{tcolorbox}[enhanced, frame hidden, sharp corners, boxsep=0pt, before skip=5pt, after skip=5pt, colback=blue!5!white]
\begin{theorem}[Under LRH, networks computing decision trees can be efficiently distilled]\label{thm:decision-tree-pac-dst}
For any $\epsilon,\delta \in (0,1)$, there is an algorithm that $(\epsilon,\delta)$-distills from $\NNdectreesLRH$ to $\dectrees$ and runs in $\poly(d,m,1/\epsilon, s, 2^r,\log(1/\delta), \tau, B)$ time and $\poly(1/\epsilon,s,\log(d /\delta), \log(\tau B))$ samples from $\cD$, where $B \geq \max_{\bx} \|\varphi_{\theta}(\bx)\|$ is an upper bound on the norm of the network representation.
\\

Furthermore, if $\cD$ is uniform over $\{+1,-1\}$, then there is an algorithm that $(\epsilon,\delta)$-distills from $\NNdectreesLRH$ to $\dectrees$ in $\poly(d,m,1/\epsilon,s,\log(1/\delta),\tau,B)$ time.
\end{theorem}
\end{tcolorbox}

\paragraph{Significance of this result} Similarly to Proposition~\ref{prop:distilling-with-juntas}, this result supports the ``\textit{Learn first, distill later}'' paradigm. If $\tau,B,m \leq \mathrm{poly}(d)$ are polynomial-size, then our distillation algorithm takes only polynomial $\mathrm{poly}(d,2^r)$ time. On the other hand, learning a decision tree from random examples is conjectured to require the much larger time $d^{\Omega(r)}$, because $r$-juntas can be encoded as decision trees of depth $r$ with size $2^r$. Thus, once again, the bulk of the compute and data is required for the learning phase, and the distillation phase can be much cheaper.

Furthermore, this result illustrates that using extra neural network structure beyond query access can aid the design of distillation algorithms that run faster than the best known query learners.

For arbitrary distributions, our distillation algorithm takes exponential time in the depth, and thus is truly polynomial-time only when $r = O(\log(s))$ It is an interesting open problem whether it is possible to obtain polynomial dependence on the depth under the LRH condition.

\subsubsection{Proof of Theorem~\ref{thm:decision-tree-pac-dst}}

A building block of our distillation procedure is a ``linear probe'' subroutine, which checks whether a function $g : \cX \to \R$ can be approximated by a low-norm linear function of a representation $\varphi : \cX \to \R^m$.

\begin{tcolorbox}[enhanced, frame hidden, sharp corners, boxsep=0pt, before skip=5pt, after skip=5pt, colback=blue!5!white]
\begin{restatable}[Linear probe subroutine]{lemma}{linearprobelemma}\label{lem:linear-probe}
Given a function $g : \cX \to [-1,1]$, a representation map $\varphi : \cX \to \R^m$ of norm bounded by $B \geq \max_{\bx} \|\varphi(\bx)\|$, a coefficient norm-bound $\tau > 0$, error tolerance parameters $\epsilon,\delta > 0$, and a distribution $\cD$, there is a subroutine $\textsc{LinearProbe}(g,\varphi,B,\tau,\epsilon,\delta,\cD)$ that runs in $\poly(1/\epsilon,\log(1/\delta),\tau,B,m)$ time and draws $\poly(1/\epsilon,\log(1/\delta),\tau, B)$ samples from distribution $\cD$
and returns 
$\mathrm{probe} \in \{\mathsf{true}, \mathsf{false}\}$ such that:

\begin{itemize}
\item If there is $\bw \in \R^m$ such that $\|\bw\| \leq \tau$ and $\E_{\bx \sim \cD}[(\bw \cdot \varphi(\bx) - g(\bx))^2] \leq \epsilon$, then $\mathrm{probe} = \mathsf{true}$ with probability $\geq 1-\delta$.

\item If for all $\bw \in \R^m$ such that $\|\bw\| \leq \tau$ we have $\E_{\bx \sim \cD}[(\bw \cdot \varphi(\bx) - g(\bx))^2] \geq 2\epsilon$, then $\mathrm{probe} = \mathsf{false}$ with probability $\geq 1-\delta$.
\end{itemize}
\end{restatable}
\end{tcolorbox}
\begin{proof}[Proof of Lemma~\ref{lem:linear-probe}]
The proof is a standard application of convex optimization and generalization bounds based on Rademacher complexity. See Appendix~\ref{app:linear-probe-proof}.
\end{proof}

With this subroutine in mind, we present the pseudocode for our distillation procedure in Algorithm~\ref{alg:distill-tree}. This algorithm proceeds in two phases. First, there is a a search phase, where we extract a set of $\ANDfun$ functions that can be efficiently represented by the network using the \textsc{LinearProbe} subroutine. Second, there is a stitching phase, where we use dynamic programming to combine these $\ANDfun$ functions that we have found to construct the best decision tree.

We hope that this two-phase blueprint will be generally useful for distillation algorithms in other settings: these algorithms could first extract a candidate set of high-level features from the network, and then distill to a model that uses these features as the backbone of its intermediate computations.

\begin{tcolorbox}[enhanced, frame hidden, sharp corners, boxsep=0pt, before skip=0pt, after skip=0pt, colback=blue!5!white, top=-8pt, left=2pt, right=2pt, bottom=5pt]
\begin{algorithm}[H]
   \caption{Distilling a neural network into a decision tree; Theorem~\ref{thm:decision-tree-pac-dst}}
   \label{alg:distill-tree}

\begin{algorithmic}[1]

    \State \begin{varwidth}[t]{\linewidth}{\bfseries Inputs:} Neural net $f_{\theta} : \{0,1\}^d \to \{0,1\}$ and its representation $\varphi_{\theta} : \{0,1\}^d \to \R^m$
    \par \hskip 4em Access to random samples from distribution $\cD$
    \par \hskip 4em Hyperparameters: $\algdepth \in \N$, $\epsilon,\delta > 0$
      \end{varwidth}
    
    \State {\bfseries Output:} Decision tree $\hat{T}$

\Statex

   \State {\color{blue} \#\# \textbf{Phase 1: Search for  $\ANDfun \mbox{s}$ that network can efficiently represent.}} 

    \State $\cS_0 \gets \{\emptyset\}$

    \For{$i=0$ \mbox{ to } $i = \algdepth-1$}
    \State \label{line:constructcPi} $\cP_i \gets \{S \in \cS_i \mbox{ s.t. } \textsc{LinearProbe}(\ANDfun_S,\varphi_{\theta},B,\tau,2^{-i-3},\frac{\delta}{2|\cS_i|\algdepth},\mathrm{Unif}\{0,1\}^d) = \mathsf{true}\}$
    \State $\cS_{i+1} \gets \bigcup_{S \in \cP_i} \textsc{Successors}(S)$ {\color{gray} \#  defined in \eqref{eq:successors-def}}
    
    \EndFor

    \State $\cS \gets \cS_0 \cup \cS_1 \cup \dots \cup \cS_{\algdepth}$

   \Statex

   \State {\color{blue}  \#\# \textbf{Phase 2: Construct best decision tree.}}
   \State $\hat{v}_S \gets \E_{\bx \sim \cD}[\ANDfun_S(\bx) (2f_{\theta}(\bx)-1)] \pm \epsilon/s$ \mbox{ \textbf{for each} } $S \in \cS$ 
   \State Return decision tree $\hat{T}$ maximizing $\mathrm{val}(\hat{T}, \hat{\bv})$, over all decision trees with $\cG_{\hat{T}} \subseteq \{\ANDfun_S \mid S \in \cS\}$
    
\end{algorithmic}
\end{algorithm}
\end{tcolorbox}

\paragraph{Details for Phase 1}

In the first phase of the algorithm, we iteratively construct a sequence of sets $\cS_0,\ldots,\cS_{\algdepth}$ of clauses, which will satisfy two properties with high probability:
\begin{itemize}
\item[(P1)] The sets contain intermediate computations of the decision tree: for any $\ANDfun_S \in \cG_T$ which is an AND function on $|S| = k$ variables where $k \leq \algdepth$, we have $S \in \cS_k$.
\item[(P2)] The size of the sets $\cS_i$ is bounded: $|\cS_i| \leq 2^{3\algdepth+2} \tau^2 B^2 d$.
\end{itemize}

In order to ensure property (P1), we iteratively grow the sets $\cS_i$ using the $\textsc{Successors}$ operation, which returns the set of clauses that could possibly be formed by taking one more step down the depth of the decision tree starting at a certain clause. For any clause $S \subseteq \{x_1,\ldots,x_d,\neg x_1,\ldots, \neg x_d\}$, the successors are:
\begin{align}\label{eq:successors-def}
\textsc{Successors}(S) = \{S \cup \ell \mid \ell \in \{x_1,\ldots,x_d,\neg x_1,\ldots, \neg x_d\} \mbox{ and } \ell \not\in S \mbox{ and } \neg\ell \not\in S\}\,.
\end{align}
For example, the successors of $S = \emptyset$ consist of all 1-clauses $\{x_1\},\ldots,\{x_d\}$ and $\{\neg x_1\},\ldots,\{\neg x_d\}$. Similarly, the successors of $S = \{\neg x_1\}$ consist of all 2-clauses where one of the literals is $\neg x_1$ and the other literal is neither $x_1$ nor $x_2$, i.e., $\{\neg x_1,x_j\}$ for all $j \neq 1$ and, $\{\neg x_1, \neg x_j\}$  for all $j \neq 1$.

In order to ensure property (P2), that the sets $\cS_i$ do not grow too large, we prune the search on each iteration by only adding the successors of clauses $S$ such that $\ANDfun_S$ that can be efficiently linearly represented by the network.  Our key lemma controls the number of distinct AND functions that can be approximately represented as low-norm linear combinations of the features in $\varphi_{\theta}$, and this means that the sizes of the sets $\cS_i$ cannot grow too large.

\begin{tcolorbox}[enhanced, frame hidden, sharp corners, boxsep=0pt, before skip=5pt, after skip=5pt, colback=blue!5!white]
\begin{restatable}[Linearly representing many ANDs requires high norm]{lemma}{andpackinglemma}\label{lem:and-packing}
Let $\cS$ be a collection of nondegenerate $k$-clauses. Let $\cG = \{\ANDfun_S \mbox{ for all } S \in \cS\}$. Suppose that $\varphi : \{0,1\}^d \to \R^m$ approximately satisfies the $\tau$-LRH with respect to $\cG$ in the sense that for $g \in \cG$ there is $\bw_g \in \R^m$ with $\|\bw_g\| \leq \tau$ such that
$$\E_{\bx \sim \mathrm{Unif}\{0,1\}^d}[(\bw_g \cdot \varphi(\bx) - g(\bx))^2] \leq 2^{-k-2}\,.$$
Then
$$|\cS| \leq 2^{3k+4} \tau^2 \E_{\bx \sim \mathrm{Unif}\{0,1\}^d}[\|\varphi(\bx)\|^2]  \,.$$
\end{restatable}
\end{tcolorbox}

\begin{proof}[Proof sketch] 
The proof of the lemma is deferred to Appendix~\ref{app:linear-and-packing-proof}. The key ingredient in the proof is to show that there is a subspace $\Omega \subseteq L^2(\{0,1\}^d)$ of function space such that if we project the set of functions in $\cG = \{\ANDfun_S \mbox{ for all } S \in \cS\}$ to $\Omega$, then $P_{\Omega} \cG$ cannot lie in a small-volume ball. We show that if $|\cS|$ is very large, then this contradicts the LRH. The subspace $\Omega$ that we choose is the subspace of degree-$k$ polynomials, and our arguments are Fourier-analytic and rely on writing each $\ANDfun_S$ function as a polynomial.
\end{proof}

Let us apply this lemma to conclude the proof of correctness of Phase 1. For each $i \in \{0,\ldots,\algdepth-1\}$, let $E_i$ be the event that:
\begin{itemize}
    \item the set $\cS_i$ contains all of the $i$-clauses in $\cG_T$ (i.e., it contains all of the intermediate computations of the decision tree at depth $i$), and
    \item all calls to $\textsc{LinearProbe}$ when forming the set $\cP_i$ satisfy the guarantees of Lemma~\ref{lem:linear-probe}.
\end{itemize}
For any $0 \leq i \leq \algdepth-1$, conditioned on events $E_0,E_1,\ldots,E_i$, it follows that:
\begin{itemize}
    \item The set $\cP_i$ contains all $i$-clauses of $\cG_T$, because $\cS_i$ contains all $i$-clause of $\cG_T$, and $\cS_i \cap \cG_T$ satisfies the $\tau$-LRH. Therefore $\cS_{i+1}$ contains all $(i+1)$-clauses of $\cG_T$, because it contains all possible successors to the $i$-clauses of $\cG_T$.
    \item For every clause $S \in \cP_i$, there is $\bw \in \R^m$ such that $\|\bw\| \leq \tau$ and $\E_{\bx \sim \cD}[(\bw \cdot \varphi(\bx) - g(\bx))^2] \leq 2^{-i-2}$. Therefore by Lemma~\ref{lem:and-packing}, we have $|\cP_i| \leq 2^{3i+4} \tau^2 B^2$. For each clause, there are at most two successors, so $|\cS_{i+1}| \leq 2d |\cP_i| \leq 2^{3i+5} \tau^2 B^2 d$.
\end{itemize}
A union bound on the event that each call to \textsc{LinearProbe} has the guarantees of Lemma~\ref{lem:linear-probe} implies:
\begin{align*}
    \P[E_{i+1} \mid E_0,\ldots,E_i] \geq 1 - \delta / (2\algdepth)\,.
\end{align*}
The base case of $\P[E_0] \geq 1 - \delta / (2R)$ is also easy to see since we start with $\cS = \{\emptyset\}$ and $\emptyset$ is the only $0$-clause. Thus, by induction on $i$ and a union bound over the probability of error over all calls of $\textsc{LinearProbe}$, we have proved that
\begin{align*}
\P[E_0,\ldots,E_{\algdepth}] \geq 1 - \delta/2\,.
\end{align*}
These events imply the correctness of Phase 1, which are the statements that $\cS_i$ contains all $i$-clauses of $\cG_T$ and each set $\cS_i$ is of bounded size:
\begin{align*}
\mbox{Properties (P1) and (P2) hold with probability at least } 1 - \delta/2\,.
\end{align*}
Notice also that under events $E_0,\ldots,E_R$ this phase takes time $\poly(2^{\algdepth}, \tau, B, d)$ since the sets are of size $|\cP_i| \leq |\cS_i| \leq 2^{3i+2} \tau^2 B^2 d$.

\paragraph{Details for Phase 2}

In the second phase, we use a well-known dynamic programming technique from \cite{guijarro1999exact,mehta2002decision} to find the best decision tree whose intermediate computations are contained in $\cS$. By property (P1) of Phase 1, we know that restricting to trees with clauses in $\cS$ is without loss of generality, since the ground truth tree's intermediate computations are contained in $\cS$.

Thus, in the discussion below, we only consider decision trees whose internal nodes are contained in the set $\cS$. For a decision tree $\tilde{T}$, each node can be identified with a clause $S$. The leaf clauses are denoted by $\mathrm{Leaves}(\tilde{T})$. For any leaf clause $S$, let $\tilde{T}(S) \in \{0,1\}$ denote the output. If the nodes of $\tilde{T}$ are contained in $\cS$, then for any $\bu \in \R^{\cS}$ we can define
\begin{align*}
\mathrm{val}(\tilde{T}, \bu) = \sum_{S \in \mathrm{Leaves}(\tilde{T})} u_S  (2\tilde{T}(S)-1)\,.
\end{align*}
If we define $\bv \in \R^S$ with $v_S = \E_{\bx \sim \cD}[\ANDfun_S(\bx) (2f_{\theta}(\bx)-1)]$, then
\begin{align*}
\mathrm{val}(\tilde{T},\bv) = \E_{\bx \sim \cD}[\sum_{S \in \mathrm{Leaves}(\tilde{T})} \ANDfun_S(\bx)(2f_{\theta}(\bx)-1)(2\tilde{T}(S)-1)] = 2\P_{\bx \sim \cD}[\tilde{T}(\bx) = f_{\theta}(\bx)] - 1\,.
\end{align*}
Therefore, if we have an approximation $\hat{\bv}$ to $\bv$ such that $|\hat{v}_S - v_S| \leq \epsilon / s$ for all $S$, we have
\begin{align*}
|\mathrm{val}(\tilde{T},\hat{\bv}) - \mathrm{val}(\tilde{T},\bv)| \leq \epsilon\,.
\end{align*}
Therefore, a size-$s$, depth-$r$ tree $\hat{T}$ that maximizes $\mathrm{val}(\hat{T},\hat{\bv})$ is an $\epsilon$-approximately optimal size-$s$ tree fitting the neural network $f_{\theta}$.

It remains to show that $\hat{\bv}$ can be constructed efficiently, and that 
this optimization can be performed efficiently. The estimation of $\hat{\bv}$ can be performed in $\poly(1/\epsilon, \log(|S| / \delta))$ samples and time from $\cD$ by a Hoeffding bound, and we have $\|\bv - \hat\bv\|_{\infty} \leq \epsilon$ with probability $\geq 1 - \delta/2$. The maximization of $\mathrm{val}(\hat{T},\hat{\bv})$ can be performed in $\poly(|S|,d,s)$ time by a dynamic program that computes, for each clause $S \in \cS$ and each tree size $0 \leq s' \leq s$, the subtree $\hat{T}_{S,s'}$ of size $s'$ rooted at $S$ that maximizes $\mathrm{val}(\hat{T}_{S,s'},\hat{\bv})$  \cite{guijarro1999exact,mehta2002decision}.

By guarantee (P2) of Phase 1, we know that $|\cS| \leq \poly(2^{\algdepth},\tau B,d)$. Therefore, putting the guarantees of Phase 1 and Phase 2 together proves the theorem for arbitrary distributions if we choose $\algdepth = r$ to be the depth of the tree -- since then total number of samples is $\poly(1/\epsilon,s,\log(\tau B/\delta))$ and runtime is $\poly(1/\epsilon,2^r,\tau B, \log(1 / \delta))$. For the uniform distribution, we can choose $\algdepth = \min(r,O(\log(s/\epsilon)))$, since there is a subtree of depth $r = O(\log(s/\epsilon))$ that $\epsilon$-approximates the original tree under this distribution. \qed

\subsubsection{Empirical evaluation}
\label{sec:empirical}

We have proved that the distillation procedure from neural networks to decision trees is efficient \textit{under the linear representation hypothesis}. It remains to check that this hypothesis is valid for trained networks that have learned decision trees. We provide experimental support for this in this subsection, by implementing Algorithm~\ref{alg:distill-tree}. We show that it can successfully distill models that have been trained on uniformly random trees with $d = 100$ and a certain depth from $\{2,3,4,5\}$. In Figure~\ref{fig:example-reconstructions} we visualize examples of these random trees on which we train our network.

For practical implementation purposes, we make one modification to our linear probe subroutine: in the construction of the set $\cP_i$ on Line \ref{line:constructcPi} we train our linear probe on 1000 samples with Adam and filter out all but the $k$ $\ANDfun$ functions where the probe has the lowest validation loss on a held-out dataset of 10000 samples. Each linear probe is trained for 100 iterations on 1000 samples.  In Figure~\ref{fig:number-linear-probes} we report the number of linear probes conducted by our algorithm with varying $k$, compared to the reconstruction accuracy of the final tree in terms of how well it approximates the true tree computed by the model. This figure shows that our algorithm only needs to probe a very small fraction of the total number of possible AND functions that what would be checked by a brute-force exhaustive search method. See \url{https://github.com/eboix/theory-of-model-distillation} for the code.

\begin{figure}[h!]

\begin{tabular}{c|c}
depth & example tree \\
\hline
2 & \includegraphics[scale=0.175]{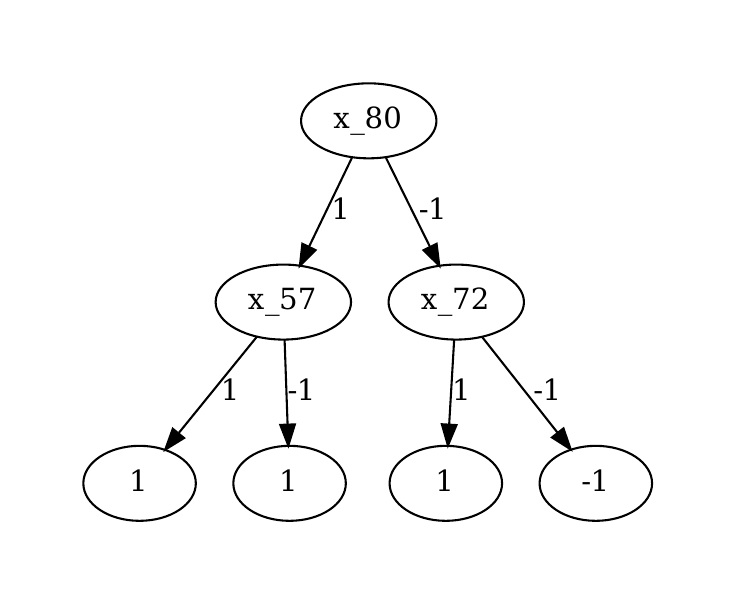} \\
3 & \includegraphics[scale=0.175]{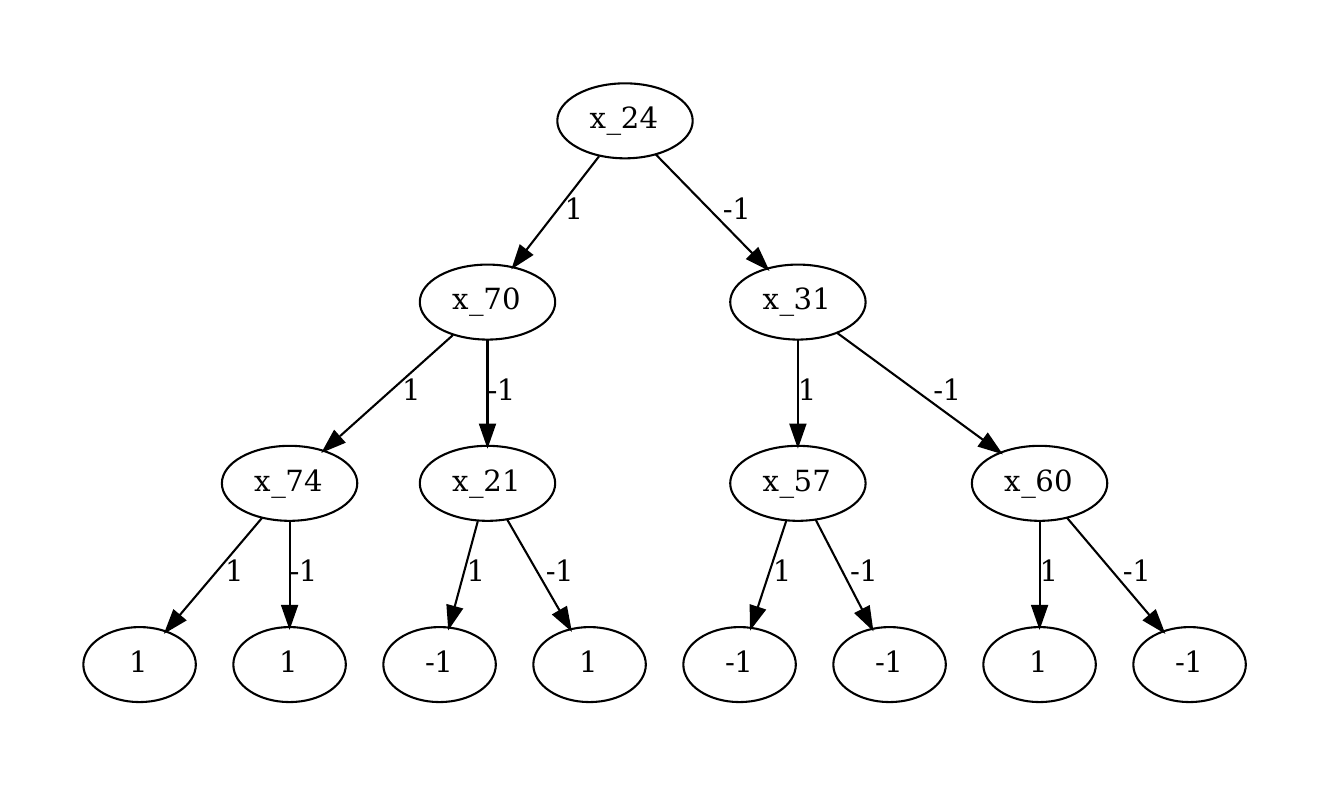} \\
4 & \includegraphics[scale=0.175]{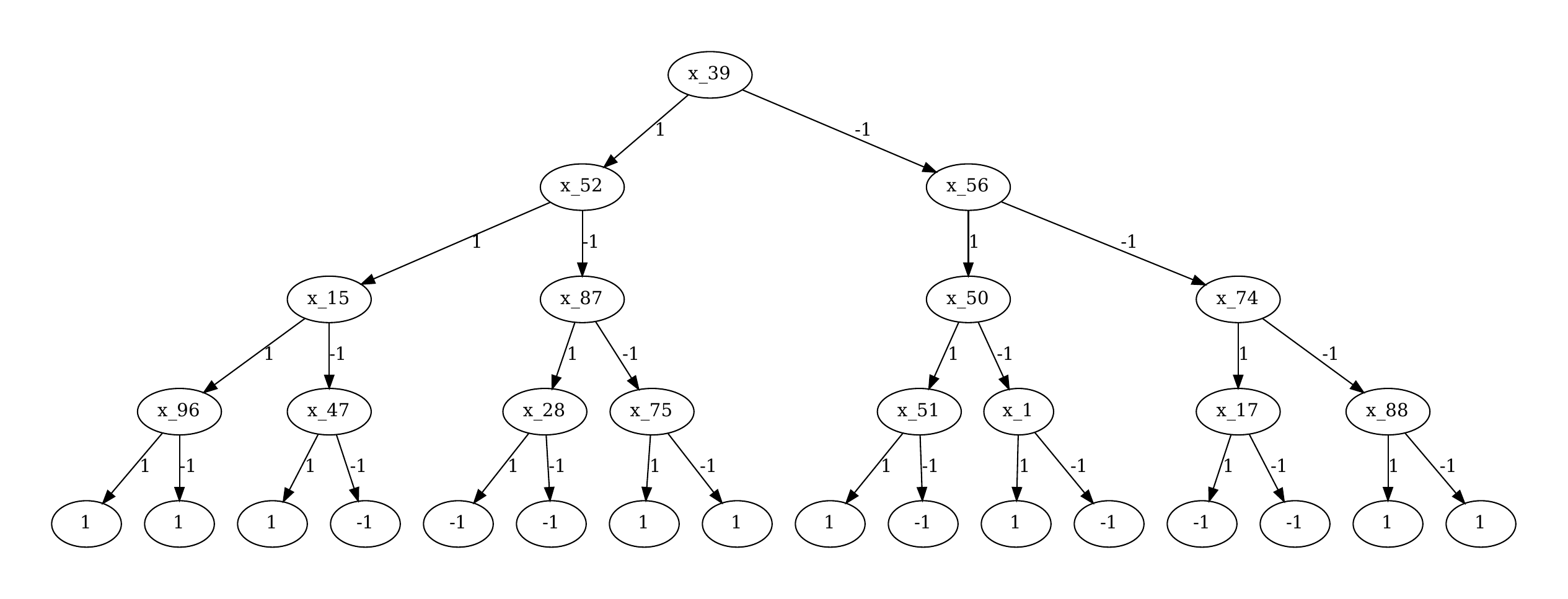} \\
5 & \includegraphics[scale=0.175]{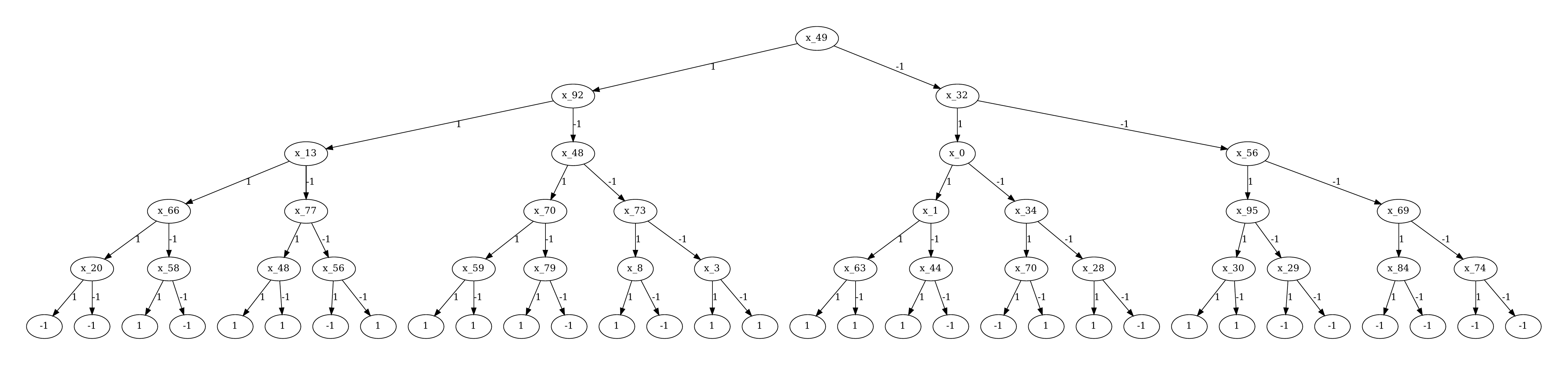} \\
\end{tabular}

\caption{Example random trees of varying depths on which we benchmark our distillation algorithm. We train a 5-layer ResNet network to learn the tree based on 100,000 random samples for the depth-2,3,4 trees and 1,000,000 random samples for the depth-5 trees. Given the trained network, we then recover from the trained network using the distillation procedure. The input space is $\{0,1\}^d$ with input dimension $d = 100$. See Figure~\ref{fig:number-linear-probes} for numerical details.}\label{fig:example-reconstructions}
\end{figure}

\begin{figure}[h!]
\centering
\begin{tabular}{ccccc}
depth & $k$ & \begin{tabular}{c} fraction of inputs \\ distillation is correct \end{tabular} & \begin{tabular}{c} average \\ number of probes \end{tabular} & \begin{tabular}{c} fraction of \\ possible probes \end{tabular} \\ 
\hline
2 & 10 & $\{ 1.00,1.00,1.00,1.00,1.00 \}$ & $2141$ & $0.107045$ \\ 
2 & 50 & $\{ 1.00,1.00,1.00,1.00,1.00 \}$ & $8901$ & $0.445028$ \\ 
2 & 100 & $\{ 1.00,1.00,1.00,1.00,1.00 \}$ & $15101$ & $0.755012$ \\ 
\textbf{2} & \textbf{  200} & \textbf{\{1.00, 1.00, 1.00, 1.00, 1.00\} } & \textbf{20001} & \textbf{ 1.000000} \\ 
\hline
3 & 10 & $\{ 1.00,1.00,1.00,0.87,1.00 \}$ & $4090$ & $0.003113$ \\ 
3 & 50 & $\{ 1.00,1.00,1.00,1.00,1.00 \}$ & $18341$ & $0.013962$ \\ 
3 & 100 & $\{ 1.00,1.00,1.00,1.00,1.00 \}$ & $32373$ & $0.024644$ \\ 
\textbf{3} & \textbf{  200} & \textbf{\{1.00, 1.00, 1.00, 1.00, 1.00\} } & \textbf{48894} & \textbf{ 0.037222} \\ 
\hline
4 & 10 & $\{ 0.87,0.81,0.88,0.74,0.75 \}$ & $6023$ & $0.000094$ \\ 
4 & 50 & $\{ 0.87,1.00,0.94,0.94,0.87 \}$ & $28035$ & $0.000438$ \\ 
4 & 100 & $\{ 1.00,1.00,1.00,1.00,0.94 \}$ & $52943$ & $0.000827$ \\ 
\textbf{4} & \textbf{  200} & \textbf{\{1.00, 1.00, 1.00, 1.00, 1.00\} } & \textbf{93519} & \textbf{ 0.001460} \\ 
\hline
5 & 10 & $\{ 0.81,0.73,0.69,0.72,0.75 \}$ & $7930$ & $0.000003$ \\ 
5 & 50 & $\{ 0.91,0.81,0.84,0.86,0.85 \}$ & $37413$ & $0.000015$ \\ 
5 & 100 & $\{ 0.97,0.94,0.91,0.94,0.87 \}$ & $71525$ & $0.000029$ \\ 
5 & 200 & $\{ 1.00,1.00,0.94,1.00,0.96 \}$ & $131617$ & $0.000053$ \\ 
5 & 500 & $\{ 1.00,1.00,0.97,1.00,0.97 \}$ & $289190$ & $0.000117$ \\ 
\textbf{5} & \textbf{  1000} & \textbf{\{1.00, 1.00, 0.97, 1.00, 1.00\} } & \textbf{539362} & \textbf{ 0.000218} \\ 
\end{tabular}
\caption{For each tree depth in $\{2,3,4,5\}$ we generate 5 random decision trees, train a depth-5 ResNet on each one with the cross-entropy loss, and distill to a decision tree. We report the results when we vary the hyperparameter $k$, which controls the size to which we prune the set of probes as we explore the space of functions efficiently represented by the network. In the third column, we report the accuracy of the distilled decision tree, which increases as the hyperparameter $k$ increases and the number of probes that the algorithm can execute increases. The average number of probes with a given depth and $k$ is reported in the fourth column, and is accurate across runs up to $\pm 0.5\%$ accuracy. The final column compares this average number of probes to the total number of possible probes of AND functions that the algorithm could make if it were brute-forcing. For depth $r$, this total number of possible probes is $\sum_{i=0}^r 2^i \binom{d}{i}$. We see in the final column that our algorithm requires only a very small fraction of the brute-force number of probes to succeed in recovering the true tree, supporting the linear representation hypothesis for networks trained on random decision trees. We leave more in-depth experiments beyond synthetic settings to future work.}\label{fig:number-linear-probes}
\end{figure}

\clearpage

\section{Computational theory: a web of reductions}\label{sec:computational}

In this section, we initiate a general computational theory of PAC-distillation. The core idea is that we should view distillation as a reduction between two hypothesis classes, because it satisfies a transitivity property. In order to define these reductions formally, we must define what it means to distill from class $\cF$ to class $\cG$ in polynomial time. In this section, we assume that each class $\cF$ is equipped with a function $\size : \cF \to \N$ that gives the representation size of a hypothesis.

\begin{tcolorbox}[enhanced, frame hidden, sharp corners, boxsep=0pt, before skip=5pt, after skip=5pt, colback=blue!5!white]
\begin{definition}[Polynomial-time distillation] 
Class $\cF$ can be distilled into class $\cG$ in polynomial  time if there is a polynomial $p(\cdot,\cdot,\cdot)$ such that for any $\epsilon,\delta \in (0,1)$ and $m \in \N$, we can $(\epsilon,\delta)$-distill from the class $\cF_m = \{f \in \cF : \size(f) = m\}$ into the class $\cG$ in less than $p(1/\epsilon,\delta,m)$ time and samples.
\end{definition}
\end{tcolorbox}

For technical reasons, we will need a slightly more general definition: distilling a sequence of classes into another sequence of classes in polynomial time. This is more general than the above definition because we can always take $\cI$ to be a singleton set.

\begin{tcolorbox}[enhanced, frame hidden, sharp corners, boxsep=0pt, before skip=5pt, after skip=5pt, colback=blue!5!white]
\begin{definition}[Polynomial-time distillation between sequences of classes] Let $\cI$ be an index set. Classes $\{\cF_{i}\}_{i \in \cI}$ can be distilled into classes $\{\cG_i\}_{i \in \cI}$ in polynomial time if there is a 
polynomial $p(\cdot,\cdot,\cdot)$ such that for any $i \in \cI$, the class $\cF_i$ can be distilled into $\cG_i$ in polynomial time bounded by $p(\cdot,\cdot,\cdot)$.
\end{definition}
\end{tcolorbox}

We use the shorthand $\{\cF_i\}_{i \in \cI} \dstinto \{\cG_i\}_{i \in \cI}$ to denote that the sequence of classes $\{\cF_i\}_{i \in \cI}$ can be distilled into the sequence of classes $\{\cG_i\}_{i \in \cI}$ in polynomial time.

\begin{tcolorbox}[enhanced, frame hidden, sharp corners, boxsep=0pt, before skip=5pt, after skip=5pt, colback=blue!5!white]
\begin{lemma}[Transitivity of polynomial-time distillation]\label{lem:distillation-transitivity} If $\{\cF_i\}_i \dstinto \{\cG_i\}_i$ and $\{\cG_i\}_i \dstinto \{\cH_i\}_i$, then $\{\cF_i\}_i \dstinto \{\cH_i\}_i$.
\end{lemma}
\end{tcolorbox}
\begin{proof}
Immediate by composing the distillation algorithms for each $i \in \cI$, and using a union bound on the probability of error and a triangle bound on the quantity of error.
\end{proof}

\begin{remark}\label{rem:agnostic-distillation-composition}
We remark that for agnostic distillation, the transitivity property in Lemma~\ref{lem:distillation-transitivity} does not hold. However, if $\cF$ can be distilled into $\cG$ and $\cG$ can be agnostically distilled into $\cH$, then $\cF$ can be agnostically distilled into $\cH$; see further discussion in Appendix~\ref{app:agnostic-distillation-composition}.
\end{remark}

Lemma~\ref{lem:distillation-transitivity} allows us to begin to build a web of reductions between hypothesis classes. This web can be viewed as giving a computational theory for which model classes are uniformly more interpretable than others -- if we can efficiently distill from class $\cF$ to class $\cG$, then we should consider the classifiers in $\cF$ to be uniformly at least as interpretable as those in $\cG$. Otherwise, given the classifier in $\cF$ we could always distill it to a corresponding classifier from $\cG$ and use that instead.\footnote{Of course, there are situations in which two model classes may be incomparable in terms of distillation, but one of the model classes may in practice be considered by some users as more interpretable than the other. However, it is less clear in these situations how to develop a mathematical theory of interpretability. Please refer back to the related work in Section~\ref{sec:related-work} for more discussion on the literature on interpretability.}

We present a very preliminary web of reductions in Figure~\ref{fig:reductions} below, which is based off of the case studies from Section~\ref{sec:case-study}. The classes involved are:
\begin{itemize}
\item Juntas, as defined in Definition~\ref{def:junta}. This is the sequence of hypothesis classes $\{\kjuntas\}_{k \in \N}$, indexed by the size of the junta $k$. Each $k$-junta is represented by a subset $S \subseteq [d]$ of the input indices of size $|S| = k$, and a truthtable $h : \{0,1\}^k \to \{0,1\}$ on those indices. We define the representation size of a junta to be $\size((S,h)) = 2^k + d$.
\item Decision trees as defined in Definition~\ref{def:dectrees}. We consider the sequence of hypothesis classes $\{\dectrees\}_{s \in \N,r \in \N}$ indexed by the number of nodes $s$ and depth $r$ of the decision tree. We define the representation complexity of a tree $T$ to be $\size(T) = s+d$, where $s$ is the number of nodes and $d$ is the input dimension.
\item Neural networks implicitly computing juntas. This is the sequence of hypothesis classes $\{\NNkjuntas\}_{k \in \N}$, indexed by the size of the junta $k$. Each neural network $f_{\theta}$ has a representation size given by $\mathrm{size}(f_{\theta}) = d + 2^k + \mbox{number of network parameters}$. Here we add the $2^k$ so that our distillation algorithms are allowed to run in exponential time in $k$ because otherwise it would be impossible to output the truthtable of the $k$-junta in some cases.

\item Neural networks implicitly computing decision trees, under the LRH. This is the sequence of hypothesis classes $\{\cup_{\tau > 0} \NNdectreesLRH\}_{s \in \N, r \in \N}$. Each neural network $f_{\theta}$ with representation $\varphi_{\theta}$ has representation size given by $\mathrm{size}((f_{\theta},\varphi_{\theta})) = d + s + \tau \cdot \max_{\bx} \|\varphi_{\theta}(\theta)\| + \mbox{number of network parameters}$. Here we add the $s$ and the dependence on $\tau$ and the maximum representation size because our distillation algorithm from Section~\ref{ssec:decisiontrees} runs in polynomial time in these parameters.
\end{itemize}

\begin{figure}[h]\label{fig:reductions}
\centering
\begin{tikzpicture}[
    node distance=2cm,
    box/.style={draw, text width=3cm, minimum height=1.5cm, align=center},
    arrow/.style={-Stealth, shorten >=2pt, shorten <=2pt},
    dashedarrow/.style={-Stealth, dashed, draw=red},
    cross/.style={path picture={\draw[red, -Stealth] (path picture bounding box.south west) -- (path picture bounding box.north east) (path picture bounding box.south east) -- (path picture bounding box.north west);}}
]

\node[box] (trees) {size-$s$, depth-$r$ decision trees};
\node[box, below=of trees] (juntas) {$k$-juntas};
\node[box, left=4cm of trees] (nn_trees) {neural networks implicitly computing size-$s$, depth-$r$ decision trees, under LRH};
\node[box, left=4cm of juntas] (nn_juntas) {neural networks implicitly computing $k$-juntas};

\draw[arrow] (nn_juntas) -- node[above] {Prop.~\ref{prop:distilling-with-juntas}} (juntas) ;
\draw[arrow] (juntas) -- node[midway, left,align=center] {Prop.~\ref{prop:juntas-to-trees}}(trees) ;
\draw[arrow] (nn_trees) -- node[midway, above,align=center]{Thm.~\ref{thm:decision-tree-pac-dst} \\ (for uniform dist.)} (trees);

\end{tikzpicture}

\caption{Simple web of reductions based on the distillation algorithms from Section~\ref{sec:case-study}. In the distillation from $k$-juntas to size-$s$, depth-$r$ trees, we have $s = 2^k$ and $r = k$. By the transitivity property of reductions, this web implies a distillation algorithm from the class of neural networks implicitly computing $k$-juntas, to the class of size-$2^k$, depth-$k$ decision trees.}

\end{figure}

The only missing element in the web is the distillation algorithm from juntas to decision trees, which is folklore:
\begin{tcolorbox}[enhanced, frame hidden, sharp corners, boxsep=0pt, before skip=5pt, after skip=5pt, colback=blue!5!white]
\begin{proposition}[Distilling trees into juntas and vice-versa]\label{prop:juntas-to-trees}
$\{\kjuntas\}_{k}$ can be $(\epsilon=0,\delta=0)$-distilled into $\{\dectrees\}_{s=2^k,r=k}$ in $\poly(2^k,d)$ time and 0 samples.
\end{proposition}
\end{tcolorbox}
\begin{proof}
Let $(S,h)$ be a $k$-junta, where $S = \{i_1,\ldots,i_k\} \subseteq [d]$ is the set of coordinates on which the $k$-junta depends, and $h : \{0,1\}^k \to \{0,1\}$ is the truth table. The junta is represented by a decision tree of depth $k$ whose $j$th layer nodes are all labeled by $i_j$, and whose $2^k$ leaf nodes are labeled by the entries of the truth table.
\end{proof}

\paragraph{Open direction} It is an open direction for future work to expand this web of reductions to include other popular classification algorithms such as (sparse) linear threshold functions, polynomial threshold functions, nearest neighbors, and logical circuits with varying sizes, depths, and gate types -- and also to explore whether these classifiers can be extracted from trained neural networks under structural assumptions such as the LRH. It is also of interest to prove that in some cases distillation is not possible, e.g., to pinpoint which of these hypothesis classes are incomparable.
\section{Statistical theory: bounds on sample complexity}\label{sec:statistical}

We now study the sample complexity of PAC-distillation, without imposing any computational constraints on the algorithms. Our main results in this section are:
\begin{itemize}
\item Distillation has a low sample complexity -- whenever $\cF$ is \textit{perfectly distillable} into $\cG$ (defined below), very few samples are needed to do so. In particular, distilling typically requires fewer samples than learning; see Section~\ref{sec:perfect-distillation-easy}.

\item Nevertheless, agnostic distillation may have a high sample complexity. This follows from results of \cite{ben2011learning} in the KLCL model of learning, and we give a new proof here; see Section~\ref{sec:agnostic-distillation-hard}.

\item Next, we give an upper bound on the sample complexity of agnostic distillation, based on the \textit{Pareto Frontier} of the classes, which shows that distilling can be much easier than learning in some settings; see Section~\ref{sec:pareto-frontier-bound}.

\item Finally, in personal communication Fan Chen and Sasha Rakhlin \cite{chenrakhlin} have kindly pointed out an equivalence between agnostic 
distillation and the General Learning Setting of Vapnik \cite{vapnik1995nature}. This implies that there is no simple analogue of the VC-dimension that characterizes the sample complexity of agnostic distillation; see Section~\ref{sec:impossibility-characterizing-sample-agnostic-distillation}.
\end{itemize}

Throughout this section, we view $f \in \cF$ and $g \in \cG$ simply as functions, ignoring how they are represented (e.g., neural network, decision tree, or truthtable), which only mattered in previous sections due to the computational constraints on the distillation algorithms.

\subsection{Whenever perfect distillation is possible, very few samples are needed}\label{sec:perfect-distillation-easy}

In this section, we prove that distillation has a very low sample complexity, \textit{whenever it is possible to distill}. The qualifier is important, since in some cases distillation might not be possible, even with unboundedly many samples.

\begin{tcolorbox}[enhanced, frame hidden, sharp corners, boxsep=0pt, before skip=5pt, after skip=5pt, colback=blue!5!white]
\begin{observation}[Distillation may be impossible, even with unbounded samples]\label{obs:distilling-may-be-impossible}
There are classes $\cF, \cG$ such that $\cF$ cannot be $(\epsilon,\delta)$-distilled into $\cG$ for any $0 \leq \epsilon < 1$ and $0 \leq \delta < 1$.
\end{observation}
\end{tcolorbox}
\begin{proof}
Let the input space be a singleton set with one element, $\cX = \{x\}$, and let $\cY = \{0,1\}$. Let $\cF = \{h_0\}$ and $\cG = \{h_1\}$, where $h_i : \cX \to \cY$ is the function defined by $h_i(x) = i$. Then the distillation algorithm must output $h_1$, but $h_1(x) \neq h_0(x)$, so $\mathrm{error}(h_1) = 1$.
\end{proof}

Therefore, we restrict our attention to the cases in which it is possible to \textit{perfectly distill}, i.e., when it is possible to carry out distillation to any given level of accuracy. We define perfect distillation below.

\begin{tcolorbox}[enhanced, frame hidden, sharp corners, boxsep=0pt, before skip=5pt, after skip=5pt, colback=blue!5!white]
\begin{definition}
The class $\cF$ is {\em perfectly distillable} into the class $\cG$ if for all $0 < \epsilon < 1$ and $0 < \delta < 1$, there is a $(\epsilon,\delta)$-distillation algorithm from $\cF$ to $\cG$.\label{def:perfectly-distillable}
\end{definition}
\end{tcolorbox}

\begin{remark}
We remark that the definition of perfect distillation does not allow the class $\cG$ to grow as we vary $\epsilon,\delta$. This means that the results in this section do not apply if, e.g., our scenario is distilling neural networks into logical circuits, and we require larger circuits as $\epsilon$ and $\delta$ decrease. Further research is required to understand the sample complexity in this case.
\end{remark}

If perfect distillation is possible, we prove that distillation requires very few samples. In the case of finite input spaces, distillation can be achieved with 0 samples.
\begin{tcolorbox}[enhanced, frame hidden, sharp corners, boxsep=0pt, before skip=5pt, after skip=5pt, colback=blue!5!white]
\begin{theorem}\label{thm:finite-perfect-distillation}
Suppose that $\cF$ is perfectly distillable into $\cG$, and the input space $\cX$ is finite. Then $\cF$ is $(\epsilon=0,\delta=0)$-distillable into $\cG$ with $0$ samples.
\end{theorem}
\end{tcolorbox}
\begin{proof}
Let $\cD$ be uniform over $\cX$. Then since $\cF$ is $(\epsilon=1/(2|\cX|), \delta=1/2)$-distillable into $\cG$, for each $f \in \cF$ there is $g \in \cG$ such that
\begin{align*}
\P_{x \sim \cD}[f(x) \neq g(x)] < 1/(2|\cX|)\,,
\end{align*}
which is possible only if $f(x) = g(x)$ for all $x$. Therefore $\cF \subseteq \cG$, and so there is $(0,0)$-distillation algorithm that simply always outputs $f \in \cG$ given $f \in \cF$.
\end{proof}

In the case of countable (potentially infinite) input spaces $\cX$, the number of samples required might be nonzero. Nevertheless, distillation still requires only $O(\log(1/\delta)/\epsilon)$ samples, which scales mildly in $\epsilon,\delta$, and does not depend on any extra structure of $\cF$ and $\cG$.

\begin{tcolorbox}[enhanced, frame hidden, sharp corners, boxsep=0pt, before skip=5pt, after skip=5pt, colback=blue!5!white]
\begin{theorem}
Suppose that $\cF$ is perfectly distillable into $\cG$, and the input space $\cX$ is countable. Then for any $0 < \epsilon < 1$ and $0 < \delta < 1$, the class $\cF$ is $(\epsilon,\delta)$-distillable into $\cG$ with $n(\epsilon,\delta) = \ceil{\log(1/\delta)/\epsilon}$ samples. Furthermore, there is a pair of classes $\cF$ and $\cG$ such that this is tight up to a constant.
\end{theorem}
\end{tcolorbox}
\begin{proof}
$\implies$. Write $\cX = \cup_{m=1}^{\infty} \cX_m$ for finite sets $\cX_1 \subseteq \cX_2 \subseteq \cdots$. A similar argument to the previous theorem shows that for any $f \in \cF$ and $m \in \cN$, there is $g_{m} \in \cG$ such that $f(x) = g_m(x)$ for all $x \in \cX_m$. Consider the distillation algorithm that takes in $\epsilon,\delta \in (0,1)$ and draws $n(\epsilon,\delta) = \ceil{\log(1/\delta)/\epsilon}$ samples $x_1,\ldots,x_n \sim \cD$, and lets $m^* = \min \{m : \{x_1,\ldots,x_n\} \subseteq \cX_m\}$, and outputs $g_{m^*}$.

We prove correctness. Define $\tilde{m}_{\epsilon} = \max \{m : \P_{x \sim \cD}[x \not\in \cX_{m_{\epsilon}}] > \epsilon \}$. Then the error of the hypothesis returned by the algorithm is:
\begin{align*}
\P_{x_1,\ldots,x_n \sim \cD}[\mathrm{error}_{\cD}(g_{m^*};f) > \epsilon] &\leq \P_{x_1,\ldots,x_n \sim \cD}[\P_{x \sim \cD}[x \not\in \cX_{m^*}] > \epsilon] \\
&< \P_{x_1,\ldots,x_n \sim \cD}[m^* \leq m_{\epsilon}] \\
&\leq \P_{x_1,\ldots,x_n \sim \cD}[x_i \in \cX_{m_{\epsilon}} \mbox{ for all } i \in [n]] \\
&< (1-\epsilon)^n \\
&< \delta\,.
 \end{align*}

$\impliedby$. Now we prove that there is a pair of classes $\cF$ and $\cG$ such that $\cF$ can be perfectly distilled into $\cG$ and such that the above sample complexity of distillation is tight up to a constant. Let $\cX = \N$ and let $\cF = \{\mbox{zero}\}$ where $\mathrm{zero} : \cX \to \{0\}$ is the function that identically outputs zero. Let $\cG = \cup_{i=1}^{\infty} \{g_i\}$, where $g_i(x) = 1(x > i)$. Then perfect distillation is possible using the above algorithm with $\cX_m = \{1,\ldots,m\}$.

On the other hand, we will show that no algorithm $\cA$ can $(\epsilon,\delta)$-distill from $\cF$ to $\cG$ in fewer than $c\log(1/\delta)/\epsilon$ samples, for an absolute constant $c > 0$ and any $0 < \epsilon,\delta < 1/2$. For any distillation algorithm $\cA$, let $\mu_{\cA}(n,\cD)$ be the distribution over $\N$ such that $\cA$ outputs $g_i$ with probability $[\mu_{\cA}(n,\cD)](i)$ when given $n$ samples from $\cD$. For any $m,\epsilon$, let $\cD_{m,\epsilon}$ be the distribution that puts probability mass $1-\epsilon$ on $1$ and $\epsilon$ on $m$. Given a number of samples $n$ and parameters $\epsilon,\delta$, let $$m^* = \inf \{m : \P_{i \sim \mu(n,\cD_{1,\epsilon})}[i > m] < \delta / 2\}\,.$$
Notice that for any $i \in \N$, the error is
\begin{align}
\mathrm{error}_{\cD_{m^*,\epsilon}}(g_i; \mathrm{zero}) = \begin{cases} \epsilon, & i < m^* \\
0, & i \geq m^* \label{eq:d-star-error}\,.
\end{cases}
\end{align}
Furthermore,
\begin{align*}
\TV(\cD_{1,\epsilon}^{\otimes n}, \cD_{m^*,\epsilon}^{\otimes n}) \leq 1 - (1-\epsilon)^n\,,
\end{align*}
so 
\begin{align}
\TV(\mu_{\cA}(n,\cD_{1,\epsilon}), \mu_{\cA}(n,\cD_{m^*,\epsilon})) \leq 1 - (1-\epsilon)^n\,. \label{eq:helpful-tv-bound}
\end{align}
Then, if we run algorithm $\cA$ on distribution $\cD_{m^*,\epsilon}$ with $n$ samples, we have
\begin{align*}
\P_{g_i \sim \cA(n,\cD_{m^*,\epsilon})}[\mathrm{error}_{\cD_{m^*,\epsilon}}(g_i; \mathrm{zero}) \geq \epsilon] &=
\P_{i \sim \mu(n,\cD_{m^*,\epsilon})}[\mathrm{error}_{\cD_{m^*,\epsilon}}(g_i;\mathrm{zero}) \geq \epsilon] \\
&= \P_{i \sim \mu(n,\cD_{m^*,\epsilon})}[i \leq m^*] \\
& \geq \P_{i \sim \mu(1,\cD_{1,\epsilon})}[i \leq m^*] -\TV(\mu(n,\cD_{m^*,\epsilon}), \mu(n,\cD_{1,\epsilon})) \\
&> \delta / 2 + (1-\epsilon)^n\,.\end{align*}
So in order to $(\epsilon,\delta)$-distill we must have $(1-\epsilon)^n < \delta / 2$, so we have to have $n = \Omega(\log(1/\delta)/\epsilon)$.
\end{proof}

\subsection{Agnostic distillation may nevertheless require a 
high number of samples}\label{sec:agnostic-distillation-hard}

In this section, we consider the sample complexity of agnostic distillation. Since the goal is to find a hypothesis in $\cG$ that approximates the best possible one in $\cG$, we do not restrict to when perfect distillation is be possible.

In contrast with the non-agnostic setting, there are classes $\cF$ and $\cG$ such that it is impossible to agnostically distill from $\cF$ to $\cG$ using any finite number of samples. This follows from a result in \cite[Section 6]{ben2011learning}, which proves the impossibility of agnostically learning certain function classes in the Known-Labeling-Classifier-Learning (KLCL) model. Since the KLCL model is a special case of distillation (where $\cF$ is the set of all possible functions on $\cX$), this hardness result extends to PAC-distillation.

Below we give an alternative proof that agnostic distillation may require a high number of samples, focusing on finite input spaces.
\begin{tcolorbox}[enhanced, frame hidden, sharp corners, boxsep=0pt, before skip=5pt, after skip=5pt, colback=blue!5!white]
\begin{theorem}[Agnostic distillation may require a high number of samples]\label{thm:agnostic-distilling-expensive}
Let $\cX$ be a finite input space of size $|\cX| = 2m$, and let $\cY = \{0,1\}$. There is a source class $\cF$ and a target class $\cG$ such that for any $0 < \epsilon \leq 1/100$ and any $0 \leq \delta \leq 1/2$ at least $n \geq m/(20000\epsilon^2)$ samples are needed to agnostically $(\epsilon,\delta)$-distill $\cF$ into $\cG$.
\end{theorem}
\end{tcolorbox}
\begin{remark}
To interpret this result, observe that agnostic distillation from $\mathcal{F}$ to $\mathcal{G}$ is no harder than agnostic learning of $\mathcal{G}$. Therefore, it is always possible with a number of samples that grows linearly with the input space size $|\cX|$. The above theorem shows that in some cases this large number of samples is needed in order to agnostically distill.
\end{remark}
\begin{proof}
Let $\cF = \{\mathrm{zero}\}$ be the singleton set consisting only of the trivial function $\mathrm{zero} : \cX \to \cY$ that always outputs 0: i.e., $\mathrm{zero}(x) = 0$ for all $x$.

Next, let $\cX = \cX_1 \sqcup \cX_2$ be a partition of the input space into two subsets of equal sizes with elements $\cX_1 = \{x_{1,1},\ldots,x_{1,m}\}$ and $\cX_2 = \{x_{2,1},\ldots,x_{2,m}\}$. Now, for each vector $\btheta \in \{1,2\}^m$, let $g_{\btheta} : \cX \to \cY$ be the hypothesis with $$g_{\btheta}(x_{i,j}) = \begin{cases} 1, & \mbox{ if } \theta_j = i \\
0, & \mbox{ otherwise} \end{cases}\,,$$
and let $\cG = \{g_{\btheta} \mid \btheta \in \{1,2\}^m\}$.
We will show that a large number of samples is needed to agnostically distill $\cF$ into $\cG$. 

Let $0 \leq \alpha \leq 1$ be a parameter that we will set later. For any $\btheta \in \{1,2\}^m$, define the distribution $\cD_{\btheta}$ over $\cX$, which assigns probability mass 
\begin{align*}
    \cD_{\btheta}(x) = \begin{cases} \frac{1}{2m} - \frac{\alpha}{2m}, & \mbox{ if } g_{\btheta}(x) = 1 \\
    \frac{1}{2m} + \frac{\alpha}{2m}, & \mbox{ if } g_{\btheta}(x) = 0\,.
    \end{cases}
\end{align*}
Notice that under distribution $\cD_{\btheta}$, the error of hypothesis $g_{\btheta'}$ is equal to
\begin{align*}
\error_{\cD_{\btheta}}(g_{\btheta'}(x)) = \frac{1}{2} + \frac{\alpha}{2} - \frac{\alpha }{m} \cdot |\{i : \theta_i = \theta_i'\}| \,,
\end{align*}
which is minimized by taking $\btheta' = \btheta$.

Let $g_{\btheta^{dst}(S)} = \Adst(S,\mathrm{zero})$ be the hypothesis returned by a distillation algorithm $\Adst$ when given a tuple of samples $S \in \cX^n$ and the concept $\mathrm{zero} \in \cF$. Let us choose $\btheta$ uniformly at random and run $\Adst$ on samples from the distribution $\cD_{\btheta}$. Then  the expected surplus error of the distilling algorithm is as follows.
\begin{align}
\E_{\btheta \sim \{1,2\}^m}&[\E_{S \sim \cD_{\btheta}^n}[\error_{\cD_{\btheta}}(g_{\btheta^{dst}(S)}) - \error_{\cD_{\btheta}}(g_{\btheta})]] \nonumber \\
&= \E_{{\btheta \sim \{1,2\}^m}}[\E_{S \sim \cD_{\btheta}^n}[\alpha (1 - |\{i : \theta_i^{dst}(S) = \theta_i\}| / m)]] \nonumber \\
&= \frac{\alpha}{m} \sum_{i=1}^m \E_{\btheta \sim \{1,2\}^m}[\P_{S \sim \cD_{\btheta}^n} [\theta_i^{dst}(S) \neq \theta_i]] \label{eq:agnostic-lower-bound-1}
\end{align}
To bound this, define $\btheta^{-i} \in \{1,2\}^m$ to be $\btheta$ but with coordinate $i$ flipped. Formally, $\theta_j^{-i} = \theta_j$ if and only if $j \neq i$. Notice that $\btheta^{-i}$ has the same distribution as $\btheta^i$, so by linearity of expectation:
\begin{align}
\mbox{Equation } \eqref{eq:agnostic-lower-bound-1} &= \frac{\alpha}{2m} \sum_{i=1}^m \E_{\btheta \sim \{1,2\}}^m[\P_{S \sim \cD_{\btheta}^n} [\theta_i^{dst}(S) \neq \theta_i] + \P_{S' \sim \cD_{\btheta^{-i}}^n}[\theta_i^{dst}(S') = \theta_i]] \label{eq:agnostic-lower-bound-2}
\end{align}
For any $\btheta$, let $\Gamma_{\btheta}$ be a coupling between $\cD_{\btheta}^n$ and $\cD_{\btheta^{-i}}^n$. Then, by linearity of expectation
\begin{align}
\mbox{Equation } \eqref{eq:agnostic-lower-bound-2} &= \frac{\alpha}{2m} \sum_{i=1}^m \E_{\btheta \sim \{1,2\}}^m[\E_{S,S' \sim \Gamma_{\btheta}} [1(\theta_i^{dst}(S) \neq \theta_i) + 1(\theta_i^{dst}(S') = \theta_i)]]  \nonumber \\ 
&\geq \frac{\alpha}{2m} \sum_{i=1}^m \E_{\btheta \sim \{1,2\}}^m[\E_{S,S' \sim \Gamma_{\btheta}} [1(S = S')]]\,. 
\label{eq:agnostic-lower-bound-3}
\end{align}
To maximize this lower bound, we should find a coupling for each $\btheta$ that maximizes the probability that the samples $S$ and $S'$ are equal. Since the distributions differ only on the probability mass that they put on $x_{i,1}$ and $x_{i,2}$, we can couple them so that $S \sim \cD_{\btheta}^n$ and $S' \sim \cD_{\btheta^{-i}}^n$ almost surely agree on all of the samples that are not equal to either $x_{i,1}$ or $x_{i,2}$. Then the coupling on the $x_{i,1}$ and $x_{i,2}$ samples reduces to coupling two Binomial random variables counting the total number of times each appears. Effectively, letting $|S|_x = |\{i : S_i = x\}|$, we have, and letting $\TV$ denote total variation distance, we have
\begin{align*}
\mbox{Equation } \eqref{eq:agnostic-lower-bound-3} \geq \frac{\alpha}{2m} \sum_{i=1}^m \E_{\btheta \sim \{1,2\}^m}[\sum_{k=0}^n \P_{S \sim \cD_{\btheta}^n}[|S|_{x_{1,i}}& + |S|_{x_{2,i}} = k]  \\
&\cdot (1-\TV(\mathrm{Bin}(k,\frac{1-\alpha}{2}), \mathrm{Bin}(k,\frac{1+\alpha}{2})))]
\end{align*}

We will use the following calculation bounding the total variation distance between two binomial distributions from \cite{mohri2018foundations}:
\begin{claim}[Lemma~3.21 of \cite{mohri2018foundations}]
Let $\Bin(k,p)$ denote the Binomial distribution with $k$ trials and probability $p$. Then for any $0 \leq \alpha < 1$, the total variation distance is bounded by:
$$1-\mathrm{TV}(\Bin(k,\frac{1-\alpha}{2}), \Bin(k,\frac{1+\alpha}{2})) \geq 
2\Phi(k+1,\alpha)\,,$$
where $\Phi(k,\alpha) = \frac{1}{4}(1 - \sqrt{1-\exp(-k\alpha^2 / (1-\alpha^2))})$, and $\Phi(\cdot,\alpha)$ is convex in the first parameter.
\end{claim}

Using this claim, we can continue to bound the expected error of the distilling algorithm. By Jensen's inequality, 
\begin{align}
\mbox{Equation } \eqref{eq:agnostic-lower-bound-3} &\geq \frac{\alpha}{m} \sum_{i=1}^m \E_{\btheta \sim \{1,2\}^m}[\sum_{k=0}^n \P_{S \sim \cD_{\btheta}^n}[|S|_{x_{1,i}} + |S|_{x_{2,i}} = k]  \cdot \Phi(k+1,\alpha)] \nonumber \\
&\geq \frac{\alpha}{m} \sum_{i=1}^m \E_{\btheta \sim \{1,2\}^m}[\Phi(\E_{S \sim \cD_{\btheta}^n}[|S|_{x_{1,i}} + |S|_{x_{2,i}} = k] +1 ,\alpha)] \nonumber \\
&= \alpha \cdot \Phi(n/m+1 ,\alpha)\,. \label{eq:agnostic-lower-bound-4}
\end{align}
Putting equations \eqref{eq:agnostic-lower-bound-1} through \eqref{eq:agnostic-lower-bound-4} together, we obtain:
\begin{align*}
\E_{\btheta \sim \{1,2\}^m}&[\E_{S \sim \cD_{\btheta}^n}[\error_{\cD_{\btheta}}(g_{\btheta^{dst}(S)}) - \error_{\cD_{\btheta}}(g_{\btheta})]] \geq \alpha \cdot \Phi(n/m+1 ,\alpha)\,.
\end{align*}
Therefore, there is a $\btheta^*$ such that when $\cD_{\btheta^*}$ is the underlying distribution the distillation algorithm has an expected surplus error of at least $\alpha \cdot \Phi(n/m+1 ,\alpha)$. Since the surplus error is always in the range $[0,\alpha]$, a Markov bound shows:
\begin{align*}
\P_{S \sim \cD_{\btheta^*}^n}[\error_{\cD_{\btheta^*}}(g_{\btheta^{dst}(S)}) - \error_{\cD_{\btheta^*}}(g_{\btheta^*}) \geq \alpha \Phi(n/m+1,\alpha)/2] \geq \Phi(n/m+1,\alpha)/2\,.
\end{align*}
If we choose $\alpha = C\epsilon \leq 1/\sqrt{2}$ and $n \leq cm/\epsilon^2$, then
\begin{align*}\Phi(n/m+1,\alpha) &= \frac{1}{4}(1-\sqrt{1-\exp(-(c/\epsilon^2 + 1)(C\epsilon)^2)/(1-(C\epsilon)^2)}) \\
&\geq \frac{1}{4}(1-\sqrt{1-\exp(-2(c/\epsilon^2)(C\epsilon)^2 - 1))}) \\
&\geq \frac{1}{4}(1-\sqrt{1-\exp(-2cC^2 - 1))})\end{align*}
Taking $C = 70$ and $c = 1/20000$ thus results in
\begin{align*}
\P_{S \sim \cD_{\btheta^*}^n}[\error_{\cD_{\btheta^*}}(g_{\btheta^{dst}(S)}) - \error_{\cD_{\btheta^*}}(g_{\btheta^*}) \geq \epsilon] \geq 1/70\,.
\end{align*}
\end{proof}
This proof shares some elements with the proof that some classes require a large number of samples to agnostically learn (Lemma~3.23 of \cite{mohri2018foundations}). This is no coincidence, since lower bounds on the number of samples needed for agnostic distillation also imply lower bounds for agnostic learnability.

\subsection{Bounding sample complexity of agnostic distillation using Pareto frontier}\label{sec:pareto-frontier-bound}

Since the sample complexity of agnostic distillation may be large in some cases, and small in others, it is natural to ask whether there is a simple combinatorial quantity that controls the sample complexity. 

In the case of learnability, it is known that a combinatorial property called the \textit{VC dimension} controls the sample complexity \cite{vapnik1974method,blumer1989learnability}.
\begin{tcolorbox}[enhanced, frame hidden, sharp corners, boxsep=0pt, before skip=5pt, after skip=5pt, colback=blue!5!white]
\begin{definition}[VC dimension \cite{vapnik1971uniform}] Let 
the label alphabet be $\cY = \{0,1\}$. A function class $\cH \subseteq \cY^{\cX}$ is said to shatter the set of inputs $\{x_1,\ldots,x_m\} \subseteq \cX$ if for each labeling of these inputs $\bsigma \in \{0,1\}^m$ there is a function $h_{\bsigma} \in \cH$ such that $(h_{\bsigma}(x_1),\ldots,h_{\bsigma}(x_m)) = \bsigma$.

The VC dimension $\VC(\cH)$ of the function class $\cH$ is the size of the greatest shattered subset of inputs (or infinite, if there is no greatest set.)
\end{definition}
\end{tcolorbox}
The VC dimension gives both upper and lower bounds for the number of samples needed to agnostically learn 
\cite{vapnik1974method,blumer1989learnability}. We follow the presentation of \cite{anthony1999neural}.

\begin{tcolorbox}[enhanced, frame hidden, sharp corners, boxsep=0pt, before skip=5pt, after skip=5pt, colback=blue!5!white]
\begin{proposition}[VC dimension controls sample complexity of agnostic learning; Theorem~5.4 in \cite{anthony1999neural}]\label{prop:vc-agnostic-learning}
There are constants $c_1,c_2 > 0$ such that for any $0 < \epsilon < 1/40$ and $0 < \delta < 1/20$ the minimum number of samples $n(\epsilon,\delta)$ needed to agnostically $(\epsilon,\delta)$-learn hypothesis class $\cG$ is:
\begin{align*}
\frac{c_1}{\epsilon^2}\left(\VC(\cG) + \ln(1/\delta)\right) \leq n(\epsilon,\delta) \leq \frac{c_2}{\epsilon^2} \left(\VC(\cG) + \ln(1/\delta)\right)\,.
\end{align*}
If $\VC(\cG) = \infty$, then $\cG$ is not agnostically learnable with finitely many samples.
\end{proposition}
\end{tcolorbox}

In this section, we show that the sample complexity of distillation is bounded above by a natural combinatorial quantity called the ``VC dimension of the Pareto frontier'', which can be much smaller than the VC dimension that controls the sample complexity of learning. However, we prove that this is not tight -- there is no corresponding lower bound based on the VC dimension of the Pareto frontier.

First, in Lemma~\ref{lem:reduce-to-binary-singleton}, we reduce distillation without loss of generality to the case where the label alphabet is binary and the source class is a singleton set.\footnote{The reduction breaks outside of the classification setting, as discussed in Section~\ref{sec:extensions}.} The relevant notation for our reduction is as follows.
For a function $f \in \cY^{\cX}$ and a class of functions $\cG \subseteq \cY^{\cX}$, define the binary-valued class of functions that represents the error pattern of the hypotheses with respect to $f \in \cF$: $$f \oplus \cG = \{f \oplus g \mid g \in \cG\} \subseteq \{0,1\}^{\cX}, \mbox{ where } (f \oplus g)(x) = 1(f(x) \neq g(x)).$$ Also, let $\mathrm{zero} : \cX \to \{0\}$ be the function that identically outputs 0.

\begin{tcolorbox}[enhanced, frame hidden, sharp corners, boxsep=0pt, before skip=5pt, after skip=5pt, colback=blue!5!white]
\begin{lemma}[Reduction to binary labels and singleton source class]\label{lem:reduce-to-binary-singleton}
The following are equivalent:
\begin{enumerate}
    \item[(a)] Class $\cF$ is $(\epsilon,\delta)$-distillable into class $\cG$ in $n$ samples.
    \item[(b)] For all $f \in \cF$, class $\{f\}$ is $(\epsilon,\delta)$-distillable into class $\cG$ in $n$ samples.
    \item[(c)] For all $f \in \cF$, class $\{\mathrm{zero}\} \subseteq \{0,1\}^{\cX}$ is $(\epsilon,\delta)$-distillable into class $f \oplus \cG \subseteq \{0,1\}^{\cX}$ in $n$ samples.
\end{enumerate}
Furthermore, the same is true for agnostic distillation instead of distillation.
\end{lemma}
\end{tcolorbox}
\begin{proof}
(b) $\implies$ (a): For each $f \in \cF$, let $\Adstf$ be an algorithm distilling $\{f\}$ into $\cG$. Then the algorithm $\Adst$ that outputs $\Adst(S,f) := \Adstf(S)$ distills $\cF$ into $\cG$.

(a) $\implies$ (b): Conversely, given $\Adst(S,f)$, we can define $\Adstf(S) := \Adst(S,f)$ for each $f$.

(b) $\implies$ (c): Note that $\error_{\cD}(g;f) = \error_{\cD}(f \oplus g;\mathrm{zero})$. So, given algorithm $\Adstf$ that distills $\{f\}$ into $\cG$, the algorithm $\Adstf'(S) :=  f \oplus \Adstf(S)$ distills $\{\mathrm{zero}\}$ into $f \oplus \cG$.

(c) $\implies$ (b): Given $\Adstf'$, we can let $\Adstf(S) := ``\mbox{output any } g \in \cG \mbox{ s.t. } f \oplus g = \Adstf'(S)$''.
\end{proof}

We have reduced the statistical problem of distillation to the case where the source class consists of only the zero function, and the target class represents the error pattern of hypotheses with respect to this function.

Under this lens, one can see that some target functions have strictly more errors than others with respect to the source concept. Thus, an optimal distillation algorithm will always output a target function whose error pattern does not dominate any of the other target functions. Formally, we can reduce to learning with any Pareto-dominating set of target functions:

\begin{tcolorbox}[enhanced, frame hidden, sharp corners, boxsep=0pt, before skip=5pt, after skip=5pt, colback=blue!5!white]
\begin{definition}
Let the source class be $\cF = \{\mathrm{zero}\} \subseteq \{0,1\}^{\cX}$ and let the target class be $\cG \subseteq \{0,1\}^{\cX}$. We define the Pareto partial order $\preceq_{Par}$ between pairs of functions $g,h \in \cG$. We say that $g \preceq_{Par} h$ if whenever $g(x) = 0$ then also $h(x) = 0$. A subset $\cH \subseteq \cG$ is said to Pareto-dominate $\cG$ if for all $g \in \cG$ there is $h \in \cH$ such that $g \preceq_{Par} h$.
\end{definition}
\end{tcolorbox}

\begin{tcolorbox}[enhanced, frame hidden, sharp corners, boxsep=0pt, before skip=5pt, after skip=5pt, colback=blue!5!white]
\begin{lemma}[Reduction to Pareto frontier of target functions]\label{lem:reduce-to-pareto}
For any Pareto-dominating subset $\cH \subseteq \cG$, the following are equivalent:
\begin{enumerate}
    \item[(a)] Class $\cF = \{\mathrm{zero}\}$ is
$(\epsilon,\delta)$-distillable into class $\cG$ in $n$ samples.
    \item[(b)] Class $\cF = \{\mathrm{zero}\}$ is
$(\epsilon,\delta)$-distillable into class $\cH$ in $n$ samples.
\end{enumerate}
Furthermore, the same is true for agnostic distillation instead of distillation.
\end{lemma}
\end{tcolorbox}
\begin{proof}
(a) $\implies$ (b): by definition for any $g \in \cG$ there is $h \in \cH$ such that for all distributions $\cD$ we have $\mathrm{error}_{\cD}(h) \leq \mathrm{error}_{\cD}(g)$. Thus an algorithm distilling $\cF$ with $\cG$ can be converted to outputting a corresponding function in $\cH$ instead without lowering the error.

(b) $\implies$ (a): since $\cH \subseteq \cG$. For agnostic distillation, note the optimum is in $\cH$.
\end{proof}

The above reductions allow us to upper-bound the sample complexity of agnostic distillation in terms of a new combinatorial quantity: 
the VC dimension of the Pareto frontier.\footnote{We remark that Theorem~\ref{thm:pareto-vc-upper-bound} does not apply to non-agnostic distillation, which may be impossible in some cases. The example in Observation~\ref{obs:distilling-may-be-impossible} demonstrates that distilling may be impossible even when $\VCpar(\cF;\cG) = 1$.}

\begin{tcolorbox}[enhanced, frame hidden, sharp corners, boxsep=0pt, before skip=5pt, after skip=5pt, colback=blue!5!white]
\begin{theorem}[The VC dimension of the Pareto frontier upper-bounds the sample complexity of agnostic distillation]\label{thm:pareto-vc-upper-bound}
For any source class $\cF \subseteq \cY^{\cX}$ and any target class $\cG \subseteq \cY^{\cX}$, define
$$\VCpar(\cF;\cG) = \max_{f \in \cF} \inf_{\cH} \VC(f \oplus \cH)\,,$$
where the infinum is over $\cH \subseteq \cG$ such that $f \oplus \cH$ Pareto-dominates $f \oplus \cG$.

Then there is a constant $C > 0$ such that for any $0 < \epsilon < 1/40$ and $0 < \delta < 1/20$, it is possible to agnostically $(\epsilon,\delta)$-distill class $\cF$ into class $\cG$ in $n(\epsilon,\delta)$ samples, where
$$n(\epsilon,\delta) \leq \frac{C}{\epsilon^2} (\VCpar(\cF;\cG) + \ln(1/\delta))\,.$$
\end{theorem}
\end{tcolorbox}
\begin{proof}
If $\VCpar(\cF;\cG)$ is infinite, then the bound is vacuous and there is nothing to prove. If not, for each $f \in \cF$, there is $\cH_f \subseteq \cG$ such that $\VC(f \oplus \cH_f) \leq \VCpar(\cF;\cG)$ and $f \oplus \cH_f$ Pareto-dominates $f \oplus \cG$.
By Proposition~\ref{prop:vc-agnostic-learning}, for each $f \in \cG$ the class $f \oplus \cH_f$ can be agnostically $(\epsilon,\delta)$-learned with $n(\epsilon,\delta)$ samples. By ignoring the label information, agnostic distillation reduces to agnostic learning so $\{\mathrm{zero}\}$ can be agnostically $(\epsilon,\delta)$-distilled into $f \oplus \cH_f$ with $n(\epsilon,\delta)$ samples. Combining this with the reductions in Lemmas~\ref{lem:reduce-to-binary-singleton} and \ref{lem:reduce-to-pareto} concludes the proof.
\end{proof}

For agnostic distillation, the quantity $\VCpar(\cF;\cG)$ matches well with the sample complexity in the examples that we have encountered so far:
\begin{itemize}
    \item When we have equal source and target classes $\cF = \cG$, distillation can be done in 0 samples as we saw in Theorem~\ref{thm:finite-perfect-distillation}. In this case notice that $\VCpar(\cF;\cG) \leq \max_{f \in \cF} \inf_{\cH} \VC(f \oplus \cH) = \max_{f \in \cF} \VC(\{\mathrm{zero}\}) = 0$.
    \item In the example of Theorem~\ref{thm:agnostic-distilling-expensive} for when agnostic distillation may require a high number of samples, $\cF = \{\mathrm{zero}\}$, and $\cG$ is large and all its functions are on the Pareto frontier -- i.e., there is no proper subset $\cH \subseteq \cG$ such that $\cH$ Pareto-dominates $\cG$. Thus, $\VCpar(\cF;\cG) = \VC(\cG) = m$, which is large.
\end{itemize}

In light of these examples, it is natural to conjecture that the VC dimension of the Pareto frontier of functions will fully characterize the sample complexity of agnostic distillation. However, this conjecture turns out to be false, as we show below.

\begin{tcolorbox}[enhanced, frame hidden, sharp corners, boxsep=0pt, before skip=5pt, after skip=5pt, colback=blue!5!white]
\begin{theorem}[The VC dimension of the Pareto frontier does not characterize the sample complexity of distillation]
There are classes $\cF,\cG$ such that $\VCpar(\cF;\cG) = \infty$, but for any $0 < \epsilon < 1$, $0 < \delta < 1$ it is possible to $(\epsilon,\delta)$-distill class $\cF$ into class $\cG$ using $0$ samples.
\end{theorem}
\end{tcolorbox}
\begin{proof}
Let $\cF = \{\mathrm{zero}\}$, let $\cX = \{1,2\} \times \N$, and $\cY = \{0,1\}$. For any subsets $S,T,U \subseteq \N$, let $g_{S,T,U} : \cX \to \cY$ be the hypothesis with
\begin{align*}
g_{S,T,U}((1,i)) = 1(i \in S)\,, \quad g_{S,T,U}((2,i)) = 1(i \in T) \quad g_{S,T,U}((3,i)) = 1(i \in U)\,, \mbox{ for all } i \in \N\,.
\end{align*}
Then let $\cG$ be the hypothesis class 
\begin{align*}
\cG = \cG_1 \cup \cG_2, \quad \mbox{ where } \quad  \cG_1 = \bigcup_{\substack{S \subset \N \\ S \neq \emptyset \\ j \in \N}} \{g_{S,\{j\},\emptyset}\}, \quad \cG_2 = \bigcup_{j \in \N} \{g_{\emptyset,\emptyset,\{j\}}\}\,.
\end{align*}
It can be seen that $g_{S,T,U}$ Pareto-dominates $g_{S',T',U'}$ if and only if $S \subseteq S'$, $T \subseteq T'$, and $U \subseteq U'$ and this containment is strict for one of $S,T,U$. All of the functions in $\cG$ are Pareto-incomparable, so there is no proper subset of $\cG$ that Pareto-dominates $\cG$.

Thus $\VCpar(\cF;\cG) = \VC(\cG) \geq \VC(\cG_1) = \infty$, since the $\{(1,1),(1,2),\ldots,(1,m)\}$ is shattered by $\cG_1$ for all $m$.

Finally, given $0 < \epsilon, 0 < \delta$ consider the distillation algorithm $\Adst$ that outputs $g_{\emptyset,\emptyset,\{j\}}$ by picking $j$ uniformly at random from the interval $\{1,\ldots,\ceil{1/(\epsilon \delta)}\}$. For any distribution $\cD$ over the inputs, we have $\E_j[\mathrm{error}_{\cD}(g_{\emptyset,\emptyset,\{j\}})] = \E_j[\P_{x \sim \cD}[x = (3,j)]] \leq \epsilon\delta$. So, by a Markov bound, $\P_j[\mathrm{error}_{\cD}(g_{\emptyset,\emptyset,\{j\}}) \geq \epsilon] \leq \delta$. Thus, this algorithm $(\epsilon,\delta)$-distills $\cF$ into $\cG$ in 0 samples.
\end{proof}

To summarize, we have shown that, for statistical purposes, distillation reduces to considering the Pareto-dominating subsets of $\cG$ and also to the case $\cF = \{\mathrm{zero}\}$ and $\cY = \{0,1\}$. We have also shown sample complexity upper bounds for agnostic distillation in terms of a new combinatorial quantity 
$\VCpar$. However, we have also shown that $\VCpar$ fails to characterize the sample complexity of distillation. 

\subsection{Impossibility of characterizing sample complexity of agnostic distillation}\label{sec:impossibility-characterizing-sample-agnostic-distillation}

Is there a simple combinatorial property that characterizes the sample complexity of agnostic distillation? It turns out that the answer is no, in a certain technical sense. In Lemma~\ref{lem:reduce-to-binary-singleton}, we have reduced to the case where $\cF = \{\mathrm{zero}\}$ and $\cY = \{0,1\}$. Thus, understanding the sample complexity of agnostic distillation reduces to understanding when, given $m$ i.i.d. samples from $\cD$, we can find a hypothesis $g \in \cG$ such that with probability at least $1-\delta$, we have
\begin{align*}
\E_{x \sim \cD}[g(x)] - \max_{g^* \in \cG} \E_{x \sim \cD}[g^*(x)] \leq \epsilon\,.
\end{align*}
In personal communication, Fan Chen and Sasha Rakhlin \cite{chenrakhlin} have observed that this is equivalent to the ``Expectation Maximization Problem'' (EMX) studied in \cite{ben2017learning}, which in turn is equivalent to learning in the ``General Learning Setting'' of Vapnik \cite{vapnik1995nature,shalev2010learnability}. By \cite[Theorem 8]{ben2017learning}, this implies that there is no ``finite character property'' (given by a first order formula provable by ZFC) analogous to the VC-dimension that characterizes the sample complexity of agnostic distillation.

\section{Extensions and future directions}\label{sec:extensions}

The contributions of this paper are (1) to initiate the study of a general computational and statistical theory of model distillation, and (2) to present case studies on distilling neural networks into more lightweight classifiers.

We hope that this work encourages further theoretical and empirical exploration into distillation. Below, we outline some of the many open directions for future work. 

\paragraph{Extending the definition of distillation} The definition of PAC-distillation can be easily extended beyond the simple case of classification considered in this work.
\begin{itemize}
\item \textit{Beyond classification}. In this paper, we considered classifiers with the zero-one loss \eqref{eq:error-of-hypothesis}. We could extend our definitions the regression setting, where the output space is the real numbers, and other losses such as the mean-squared error loss are more natural.

\item \textit{Promise problems}. In this paper, $\cD$ was an adversarially-chosen distribution.\footnote{Apart from some cases when we considered the uniform distribution.} One can extend the PAC-distillation framework to cases where we are promised that the input distribution has some constraints and is not worst-case. One example is if we assume that the neural networks in our source class satisfy an approximate version of the LRH, as in Definition~\ref{def:approx-lrh} in the appendix.

\item \textit{Testing out-of-distribution behavior}. PAC-distillation is defined with matching train and test distributions. Is there a corresponding definition for when there is distribution shift? How does the computational and statistical complexity change in that setting?

\end{itemize}

\paragraph{Basic statistical and computational theory} Fundamental open questions include:
\begin{itemize}
\item \textit{Growing the web of reductions}. The web in Figure~\ref{fig:reductions} could be expanded to include other popular classifiers, such as nearest neighbors, linear and polynomial threshold functions, and logical circuits of varying sizes, depths, and gate types -- as well as whether these classifiers can be extracted from trained neural networks under structural assumptions such as the LRH. 

\item \textit{Statistical-computational gaps for distillation}. Are there natural cases where it is statistically possible to distill, but it is computationally hard?
\end{itemize}

\paragraph{New provably-efficient distillation algorithms for neural networks}

\begin{itemize}

\item \textit{Neural network structure beyond LRH}. What other kinds of structure in trained neural networks can we use to distill efficiently, beyond query access to the network, and the Linear Representation Hypothesis (LRH)? One possibility is to use low-rank weight structure in a trained network \cite{boix2023transformers}, which hints that internally the network may be well-approximated by a composition of multi-index functions \cite{abbe2022merged,damian2022neural}. Another possibility is to use the attention structure in a transformer, which is often visually inspected for insights into the network's inner workings \cite{vig2019multiscale}.\footnote{Nevertheless, recent work \cite{wen2023transformers} suggests that attention patterns should not be interpreted by themselves without considering how they interact with the feedforward components of the transformer.} Yet another possibility is to use internal sparse activation structure \cite{andriushchenko2023sgd} to distill. It seems that there can be a synergy between research into how neural networks learn, and new distillation algorithms that exploit that research.

\item \textit{Obtaining polynomial time in the depth for distilling decision trees}. Our algorithm in Theorem~\ref{thm:decision-tree-pac-dst} for distilling decision trees from neural networks runs in true polynomial time in all parameters when the input distribution is uniform. Is it possible to reduce the current $2^{O(r)}$ dependence on the depth to $\mathrm{poly}(r)$ dependence for arbitrary non-uniform distributions?

\item \textit{Distilling into circuits and more expressive classes}. Can we extend the ideas in the distillation algorithm from neural networks into trees so that we instead distill into a class of shallow logical circuits, which are more expressive than decision trees and may better approximate the trained neural network? Beyond small, shallow circuits, it is interesting to distill into more expressive and ``human-interpretable'' classes, such as, e.g., small Python programs.
\end{itemize}

\paragraph{Distilling foundation models} Scaling these methods to foundation models such as LLMs seems like it will pose a significant engineering challenge, and require new ideas.

\begin{itemize} \item \textit{Into which class of models can we distill an LLM?} Classical hypothesis classes, such as decision trees and very small circuits, are almost certainly insufficient to distill a 
LLM. One main obstacle is that LLMs are able to memorize a large amount of information which cannot be encoded into a limited-size circuit or decision tree. Therefore, a possible choice is classifiers that consist of a large memory bank, and a circuit that has some limited access to it and also has sparse activation patterns (i.e., one can think of the neurons in the circuits as corresponding to concepts, and 
only a bounded number of concepts is active in any sentence). Another possibility is a hypothesis class again consisting of a large key-value memory, but now coupled with a small RASP program \cite{weiss2021thinking} or other textual program. Yet another possibility is a target class consisting of a large logical circuit where a significant fraction of the nodes correspond to ``human-understandable'' concepts (e.g., words from a dictionary, or concepts encoded by a more trusted LLM). Another consideration is that it may be possible to distill the higher layers of the LLM, which may be in charge of higher-level reasoning, while maintaining the lower-level layers. One benefit of these distillations might be easier discovery and debugging of reasoning errors in distilled LLMs.
\end{itemize}

\section*{Acknowledgements}
I am very grateful to Fan Chen and Sasha Rakhlin, who pointed out the connection between the statistical complexity of distillation and the statistical complexity in the General Learning Setting of Vapnik \cite{chenrakhlin}. I am also very grateful to Emmanuel Abbe, Kiril Bangachev, Guy Bresler, Melihcan Erol, and Nati Srebro for helpful research discussions. I am also thankful to Guy Bresler for feedback on the exposition. This work was supported by an NSF Graduate Research Fellowship under NSF grant 1745302.

\bibliography{references}
\bibliographystyle{alpha}

\appendix 

\addtocontents{toc}{\protect\setcounter{tocdepth}{1}}

\section{Evidence for the linear representation hypothesis}\label{app:lrh-evidence}

We discuss evidence for the linear representation hypothesis (LRH), which is that after training a neural network's representation map $\varphi_{\theta} : \cX \to \R^m$ can efficiently represent high-level features and intermediate computations linearly \cite{mikolov2013linguistic,vargas2020exploring,elhage2022toy}. This hypothesis motivated our distillation algorithm in Sections~\ref{ssec:decisiontrees}. %
 The definition of the LRH is restated below:
\lrhdef*

We should mention that it is more natural to hypothesize the LRH in its approximate form. This variant was omitted from the main text for simplicity, but the guarantees for the decision tree algorithm could be extended to hold under the approximate LRH as well.
\begin{definition}[Approximate linear representation hypothesis]\label{def:approx-lrh}
Let $\cG$ be a collection of ``high-level'' functions, $g : \cX \to \R$ for each $g \in \cG$. For any $\tau > 0$ and $\epsilon > 0$, the $(\tau,\epsilon)$-LRH under distribution $\cD$ is the hypothesis that for all $g \in \cG$ there is a coefficient vector $\bw_g \in \R^m$ such that $\|\bw_g\| \leq \tau$ and
\begin{align*}
\E_{\bx \sim \cD}[(\varphi(\bx) \cdot \bw_g(\bx) - g(\bx))^2] \leq \epsilon\,.
\end{align*}
\end{definition}

\noindent There is quickly growing empirical and theoretical evidence for the LRH.

\paragraph{Empirical evidence for LRH} The LRH originated in work on word embeddings \cite{mikolov2013linguistic}, which discovered that linear subspaces of the embedding space encode semantic properties. This resulted in a ``vector'' algebra between embeddings, with relationships such as $\varphi(\mbox{``queen''}) \approx \varphi(\mbox{``king''}) - \varphi(\mbox{``man''}) + \varphi(\mbox{``woman''})$ that could be used to solve analogy problems.

Subsequently, this idea of semantic content was used to study neural networks via linear probes \cite{alain2016understanding} of their internal representations. Various high-level features are encoded linearly. These include, but are not limited to:
\begin{itemize}
    \item text sentiment in large language models \cite{tigges2023linear}\item syntactic content of a sentence in a machine-translation model \cite{shi2016does,conneau2018you}
    \item the parsing tree in a model trained on a Context-Free-Grammar \cite{allen2023physics}\item the Othello board's states in a language model trained to play based only off of move sequences 
\cite{li2022emergent,nanda2023emergent}
\item the truth or falsehood of a statement in a large language model \cite{marks2023geometry}
\item content about a scene described in text \cite{li2021implicit}
\item learned factual knowledge in pretrained transformers 
\cite{dai2021knowledge,zhu2023physicsb}
\item semantic WordNet \cite{fellbaum1998wordnet}
relations between words  \cite{aspillaga2021inspecting}
\end{itemize}

The LRH has also been used to control models. One application is erasing concepts from trained models \cite{bolukbasi2016man,vargas2020exploring,ravfogel2022linear,ravfogel2023log,belrose2023leace}, by projecting the internal representation orthogonal to the direction in which the concept is represented. Another application in 
\cite{wang2023concept} 
shows that the LRH can be used to extract concepts out of representations to control image generation with diffusion model. Additionally, 
\cite{hernandez2023inspecting} give a method to edit representations of entities based on linear addition of an attribute vector to the entity's representation. Additionally, \cite{park2023linear} and suggest a linear transformation of the representation space to disentangle concepts and similarly edit attributes.

Finally, there is evidence that the same collection of linear features might be universally learned by distinct networks trained on the same dataset. Indeed, differently-initialized networks appear to converge to functionally similar representations when trained, in the sense that these representations behave similarly under linear transfer learning \cite{boix2022gulp}.

\paragraph{Theoretical evidence for LRH} The LRH has been proved in certain settings.

\begin{itemize}

\item \textit{Linear structure in word embeddings}: Under a generative model of text, which encodes topic structure \cite{arora2016latent} proves that the PMI method 
\cite{church1990word} constructs word embeddings that satisfy the LRH, in these sense that there are directions of space that encode concepts as observed in
\cite{mikolov2013linguistic}.

    \item \textit{Juntas and multi-index functions}: The LRH has been proved for neural networks that have been trained by gradient-based methods to learn juntas or multi-index functions\footnote{See Definition~\ref{def:junta} for juntas, and definition of multi-index functions below.} \cite{abbe2022merged,damian2022neural,ba2022high,mousavi2022neural,abbe2023sgd,
dandi2023learning,bietti2023learning}.
\begin{itemize}
    \item Suppose that $f_{\theta}$ is a neural network that has been trained by gradient-based methods to learn a junta $f_*(\bx) = h(\bx_S)$. Then the learned representation $\varphi_{\theta}$ can linearly represent \textit{any} junta depending only on the variables $\bx_S$, with a coefficient vector of norm $\tau = O_{|S|}(1)$. In contrast, before training, some of these same juntas would have required a much larger coefficient vector of norm  $\tau \geq d^{\Omega(|S|)}$. To summarize, learning one junta on the variables $S$ makes the network's representation very efficient at linearly 
representing any other junta on the same set of variables. See \cite{abbe2022merged} for more details.
    \item For multi-index functions, a similar result is known. Multi-index functions are analogous to juntas, except that they depend on a small subspace of the input instead of on a small number of input coordinates. These are functions of the form $f(\bx) = h(\bP \bx)$, where $\bP \in \R^{k \times d}$ is a projection to a $k$-dimensional subspace. In this case, a network trained on one multi-index function learns a representation that can approximate any other multi-index function on that same subspace with only a $\tau = O_k(1)$-norm vector of coefficients. See \cite{damian2022neural} for more details.
\end{itemize}

\end{itemize}

\section{Deferred technical proofs}\label{app:deferred-proofs}

\subsection{Lemmas for decision-tree distillation}

\subsubsection{Proof of linear probe subroutine, Lemma~\ref{lem:linear-probe}}\label{app:linear-probe-proof}

\linearprobelemma*

\begin{proof}[Proof of Lemma~\ref{lem:linear-probe}] Given i.i.d. samples $\bx_1,\ldots,\bx_n \sim \cD$ we can form the empirical loss
\begin{align*}
L_n(\bw) = \frac{1}{n} \sum_{i=1}^n (\bw \cdot \varphi(\bx_i) - g(\bx_i))^2\,,
\end{align*}
and construct $\hat\bw$ by 
optimizing $\min_{\bw \mbox{ s.t. } \|\bw\| \leq \tau} L_n(\bw)$ up to $\epsilon/4$ additive error in time $\poly(n,\epsilon,\tau,B,m)$ by Frank-Wolfe (e.g., Theorem 3.8 of \cite{bubeck2015convex}). Because of the bounds on $g,\bw,\varphi$, we can write the population loss
\begin{align*}
L(\bw) = \E_{\bx \sim \cD}[(\bw \cdot \varphi(\bx))^2]\,,
\end{align*}
in the form
\begin{align*}
L(\bw) = \E_{\bx \sim \cD}[\ell(\bw,\bx)]\,,
\end{align*}
where $\ell(\cdot,\cdot)$ is bounded by $(B\tau + 1)^2$ and is $\poly(B,\tau,1)$-Lipschitz in its first argument. Thus, by standard Rademacher complexity arguments and McDiarmid's inequality (e.g., Proposition 4.5 of \cite{bach2021learning}), it suffices to take a number of samples $n = \poly(\log(1/\delta),\epsilon,B,\tau)$ in order to guarantee that with probability at least $1-\delta$ we have $|\min_{\bw^* \in B(0,\tau)} L(\bw^*) - L_n(\hat\bw)| \leq \epsilon/3$, which allows the probe to satisfy the claimed guarantee.
\end{proof}

\subsubsection{Proof of linear representation of AND bound, Lemma~\ref{lem:and-packing}}\label{app:linear-and-packing-proof}

We prove Lemma~\ref{lem:and-packing}, on the maximum number of ANDs that can have low kernel norm for a given representation. First, we prove a simpler version of this lemma, where we show that it is impossible to represent the AND functions \textit{exactly}, without any error. We will later bootstrap this weaker version into the proof of Lemma~\ref{lem:and-packing}, which shows that is hard to even \textit{approximately} represent a large collection of AND functions.

In order to prove this lemma, we first make a definition of the index set of a clause, and prove an exact version of the lemma.

\begin{definition} The {\em index set} of a clause $S \subseteq \{x_1,\ldots,x_d,\neg x_1,\ldots, \neg x_d\}$ is the set $\cI(S) = \{i \in [d] : x_i \in S \mbox{ or } \neg x_i \in S\}$. 
\end{definition}
\begin{lemma}[Exact linear representation of many ANDs requires high norm]\label{lem:exact-and-packing}

Let $\cS$ be a set of nondegenerate $k$-clauses such that all distinct pairs $S, S' \in \cS$ have distinct index sets $\cI(S) \neq \cI(S')$. Let $\cG = \{\ANDfun_S \mbox{ for all } S \in \cS\}$. If $\varphi : \{+1,-1\}^d \to \R^m$ satisfies the $(\tau, 0)$-LRH with respect to $\cG$. Let $\beta^2 = \E_{\bx}[\|\varphi(\bx)\|^2]$. Then 

$$\beta^2\tau^2 \geq 2^{-2k-2} |\cS|\,.$$

\end{lemma}
\begin{proof}
Without loss of generality, we also use the input alphabet $\{+1,-1\}^d$ instead of $\{0,1\}^d$, since it is more notationally convenient for the Fourier-analytic techniques.

By the linear representation hypothesis, for each $S \in \cS$, there $\bw_S \in B(0,\tau)$ such that
\begin{align}
\bw_S \cdot \varphi(\bx) = \ANDfun_S(\bx)\, \mbox{ for all } \bx \in \{+1,-1\}^d\,. \label{eq:exact-lrh}
\end{align}
Define the matrix $\bW \in \R^{\cS \times m}$ with rows $\bw_S$, and the matrix $\bPhi \in \R^{m \times 2^d}$ with columns $\varphi(\bx)$ for $\bx \in \{+1,-1\}^d$. Writing equation \eqref{eq:exact-lrh} in matrix-form, we have
\begin{align}
[\bW \bPhi]_{S,\bx} = \ANDfun_S(\bx)\,. \label{eq:exact-lrh-matrix}
\end{align}

We first prove that these are quite close to linearly independent.
\begin{claim}[Matrix of features has lower-bounded singular values]\label{claim:sing-value-lower-bound}
$\bW \bPhi \bPhi^{\top} \bW^{\top} \gtrsim 2^{d-2k-2} \bI$
\end{claim}
\begin{proof}[Proof of Claim~\ref{claim:sing-value-lower-bound}]
Let $\chi_A(\bx) = \prod_{i \in A} x_i$. Define $\bV \in \R^{\binom{d}{k} \times 2^d}$ to be the matrix whose rows correspond to the degree-$k$ monomials: $\bV_{A,\bx} = \chi_{A}(\bx)$ for all $A \in \binom{[d]}{k}$. Let $\bP_k$ be the orthogonal projection matrix to the span of the rows of $\bV$, and let $\bP_k^{\perp}$ be the projection to the orthogonal complement. Since the only degree-$k$ term in the Fourier expansion of $\ANDfun_S$ is $\pm \frac{1}{2^{k+1}} \chi_{I(S)}(\bx)$, we have:
\begin{align*}
[\bW \bPhi \bP_k]_{S,\bx} = \pm \frac{1}{2^{k+1}} \chi_{\cI(S)}(\bx)\,.
\end{align*}
This implies that
\begin{align*}
\bW \bPhi \bPhi^{\top} \bW^{\top} &= \bW \bPhi (\bP_k + \bP_k^{\perp})(\bP_k + \bP_k^{\perp})^{\top} \bPhi^{\top} \bW^{\top} \\
&= \bW \bPhi\bP_k \bP_k^{\top} \bPhi^{\top} \bW^{\top} +  \bW \bPhi \bP_k^{\perp} (\bP_k^{\perp})^{\top} \bPhi^{\top} \bW^{\top} \\
&\gtrsim \bW \bPhi\bP_k \bP_k^{\top} \bPhi^{\top} \bW^{\top} \\
&= 2^{d-2k-2} \bI\,.
\end{align*}
\end{proof}

Next, we use the linear representation hypothesis (LRH) to prove that the singular values of the feature matrix are upper-bounded, which will yield a contradiction when combined with the previous claim:
\begin{claim}[Singular values upper-bounded by LRH]\label{claim:features-determinant-small}
 $\det(\bW \bPhi \bPhi^{\top} \bW) \leq (2^d \beta^2 \tau^2 / |\cS|)^{|\cS|}$.
\end{claim}
\begin{proof}[Proof of Claim~\ref{claim:features-determinant-small}]
Notice that $m \geq |\cS|$ since otherwise the rank of $\bW \bPhi$ is not large enough for Claim~\ref{claim:sing-value-lower-bound} to hold. Therefore, we can write the matrices in the LQ decomposition as $\bW = \bL \bQ$ where $\bL \in \R^{\cS \times \cS}$ is a lower-triangular matrix and $\bQ \in \R^{\cS \times m}$ has orthonormal rows. We have
\begin{align*}
\tau^2 \geq \|\bw_S\|^2 &= \sum_{i \in [m]} [\bL \bQ]_{S,i}^2 = \sum_{i \in [m]} (\sum_{S' \in \cS} L_{S,S'} \bQ_{S',i})^2 \\
&= \sum_{i \in [m]} \sum_{S',S''} L_{S,S'}L_{S,S''} Q_{S',i} Q_{S'',i} =\sum_{S',S''} L_{S,S'}L_{S,S''}  \sum_{i \in [m]}  Q_{S',i} Q_{S'',i} = \sum_{S'} L_{S,S'}^2 \geq L_{S,S}^2\,.
\end{align*}
So since $\bL$ is lower-triangular we have
$$|\det(\bL)| = |\prod_{S \in \cS} L_{S,S}| \leq \tau^{|\cS|}.$$
Furthermore, since $\bQ$ is a semi-orthogonal matrix we have $\|\bQ\| \leq 1$ so each column of $\bQ \bPhi$ has norm at most
\begin{align*}
\|[\bQ \bPhi]_{*,\bx}\| \leq \|\bPhi_{*,\bx}\| = \|\varphi(\bx)\|\,.
\end{align*}
Therefore
\begin{align*}
\tr(\bQ \bPhi \bPhi^{\top} \bQ^{\top}) &= \tr(\bPhi^{\top} \bQ^{\top} \bQ \bPhi) = \sum_{\bx} \|[\bQ \bPhi]_{*,\bx}\|^2 \leq \sum_{\bx} \|\varphi(\bx)\|^2 \leq 2^d \beta^2\,.
\end{align*}
So since $\bQ \bPhi \bPhi^{\top} \bQ^{\top}$ is p.s.d. and has dimensions $|\cS| \times |\cS|$, we have
\begin{align*}
|\det(\bQ \bPhi \bPhi^{\top} \bQ^{\top})| &\leq (\tr(\bQ \bPhi \bPhi^{\top} \bQ^{\top}) / |\cS|)^{|\cS|} \leq (2^d \beta^2 / |\cS|)^{|\cS|}\,.
\end{align*}
We conclude that
\begin{align*}
|\det(\bW\bPhi\bPhi^{\top} \bW)| &= |\det(\bL\bQ\bPhi\bPhi^{\top} \bQ^{\top}\bL^{\top})| \\
&= |\det(\bL)|^2\det(\bQ\bPhi\bPhi^{\top} \bQ^{\top}) \leq (2^d \beta^2 \tau^2 / |\cS|)^{|\cS|}\,.
\end{align*}
\end{proof}

We combine Claims~\ref{claim:sing-value-lower-bound} and \ref{claim:features-determinant-small}:
\begin{align*}
2^{(d - 2k-2)|\cS|} \leq \det(\bW \bPhi \bPhi^{\top} \bW^{\top}) \leq (2^d \beta^2 \tau^2 / |\cS|)^{|\cS|}\,.
\end{align*}
This implies 
\begin{align*}
\beta^2\tau^2 \geq 2^{-2k-2} |\cS|\,.
\end{align*}
and thus we have proved Lemma~\ref{lem:exact-and-packing}.
\end{proof}

We can now prove the lemma when the packing is only approximate.

\andpackinglemma*

\begin{proof}[Lemma~\ref{lem:and-packing}]

Without loss of generality, we use the input alphabet $\{+1,-1\}^d$ instead of $\{0,1\}^d$, since it is more notationally convenient for the Fourier-analytic techniques.

Assume without loss of generality that for all $S,S' \in \cS$, the variables on which $S$ depends are a distinct set from the variables on which $S'$ depends. This can be ensured by shrinking $|\cS|$ by a factor of at most $2^k$. Thus, it suffices to show the following:
\begin{align}
\label{eq:and-packing-approximate-distinct}
\tau^2 \E_{\bx}[\|\varphi(\bx)\|^2] \geq 2^{-2k-4} |\cS|\,.
\end{align}

By the approximate LRH, for each $S \in \cS$, there is $\bw_S \in \R^m$ and a function $h_S \in L^2(\mathrm{Unif}\{+1,-1\}^d)$ such that
\begin{align*}
h_S(\bx) = \ANDfun_S(\bx) - \bw_S \cdot \varphi(\bx), \mbox{ and } \|h_S\| \leq 2^{-k-2}\,.
\end{align*}
We use this to construct a new embedding for each variable $\psi(\bx)$ that can \textit{exactly} represent all of the ANDs in $\cS$ with low norm. This embedding is given by the concatenation:
$$\psi(\bx) := (\varphi(\bx), \frac{1}{\tau}[h_S(\bx)]_{S \in \cS}) \in \R^{m+|\cS|}.$$
We also define $$\bv_S = (\bw_S, \tau \be_S) \in \R^{m + |\cS|},$$ so that $$\bv_S \cdot \psi(\bx) = \bw_S \cdot \varphi(\bx) + h_S(\bx) = \ANDfun_S(\bx).$$
For each $S \in \cS$, we have
\begin{align*}
\|\bv_S\|^2 = \|\bw_S\|^2 + \tau^2 \leq 2\tau^2
\end{align*}
and we also have
\begin{align*}
\E_{\bx}[\|\psi(\bx)\|^2] &= \E_{\bx}[\|\varphi(\bx)\|^2] + \sum_{S \in \cS}\E_{\bx}[h_S(\bx)^2 / \tau^2] \leq \E_{\bx}[\|\varphi(\bx)\|^2] + 2^{-2k-4} |\cS| / \tau^2
\end{align*}

Thus the representation $\psi(\bx)$ satisfies the $\sqrt{2}\tau$-LRH, and by Lemma~\ref{lem:exact-and-packing} we know that
\begin{align*}
2\tau^2 (\E_{\bx}[\|\varphi(\bx)\|^2] + 2^{-2k-4} |\cS| / \tau^2) \geq 2^{-2k-2} |\cS|\,,
\end{align*}
which implies 
\begin{align*}
\tau^2 \E_{\bx}[\|\varphi(\bx)\|^2] \geq 2^{-2k-4} |\cS| / \tau^2\,,
\end{align*}
which is the claim in \eqref{eq:and-packing-approximate-distinct} that we had to show.
\end{proof}

\subsection{Composing distillation with agnostic distillation}\label{app:agnostic-distillation-composition}

We fill out the details for Remark~\ref{rem:agnostic-distillation-composition} here.

First, note that two distillations compose in a natural way.

\begin{lemma}[Composing distillation]
Suppose that $\cF$ can be $(\epsilon_1,\delta_1)$-distilled into $\cG$ in $s_1$ samples and $t_1$ time, and suppose also that $\cG$ can be $(\epsilon_2,\delta_2)$-distilled into $\cH$ in $s_2$ samples and $t_2$ time. Then $\cF$ can be $(\epsilon_1+\epsilon_2,\delta_1+\delta_2)$-distilled into $\cH$ in $\max(s_1,s_2)$ samples and $t_1 + t_2$ time.
\end{lemma}
\begin{proof}
Immediate by composing the two distillation algorithms.
\end{proof}

However, as pointed out in Remark~\ref{rem:agnostic-distillation-composition}, two agnostic distillations do not compose. On the other hand, distillation composes with agnostic distillation, as we point out here.

\begin{lemma}[Composing distillation with agnostic distillation]
Suppose that $\cF$ can be $(\epsilon_1,\delta_1)$-distilled into $\cG$ in $s_1$ samples and $t_1$ time, and suppose also that $\cG$ can be \underline{agnostically} $(\epsilon_2,\delta_2)$-distilled into $\cH$ in $s_2$ samples and $t_2$ time. Then $\cF$ can be \underline{agnostically} $(2\epsilon_1+\epsilon_2,\delta_1+\delta_2)$-distilled into $\cH$ in $\max(s_1,s_2)$ samples and $t_1 + t_2$ time.
\end{lemma}
\begin{proof}
    Compose the distillation from $\cF$ to $\cG$ with the agnostic distillation from $\cG$ to $\cH$. Let $f \in \cF$ be the input, $g \in \cG$ be the intermediate step, and $h \in \cH$ be the output. By a union bound, with probability $\delta_1+\delta_2$,
    \begin{align*}
    \mathrm{error}_{\cD}(h;f) &\leq \mathrm{error}_{\cD}(h;g) + \mathrm{err}_{\cD}(g;f) \\
    &\leq \mathrm{error}_{\cD}(h;g) + \epsilon_1 \\
    &\leq \min_{\tilde{h} \in \cH} \mathrm{error}_{\cD}(\tilde{h};g) + \epsilon_2 + \epsilon_1 \\
    &\leq \min_{\tilde{h} \in \cH} \mathrm{error}_{\cD}(\tilde{h};f) + 2\epsilon_1 + \epsilon_2\,.
    \end{align*}
\end{proof}

\end{document}